\documentclass[letterpaper]{article} 
\usepackage{aaai23}  
\usepackage{times}  
\usepackage{helvet}  
\usepackage{courier}  
\usepackage[hyphens]{url}  
\usepackage{graphicx} 
\usepackage{natbib}  
\usepackage{enumitem}
\usepackage{caption} 
\usepackage{amsmath,amsfonts,amssymb,amsthm}
\usepackage{mathrsfs}
\allowdisplaybreaks

\usepackage{amsmath,amsfonts,amssymb}
\usepackage{xcolor}
\usepackage{algorithmic}
\usepackage[linesnumbered,ruled,vlined]{algorithm2e}
\graphicspath{{images/}}
\usepackage{mathtools,dsfont}
\usepackage{psfrag,graphicx}
\usepackage{comment,balance}
\usepackage{epstopdf}
\usepackage{epsfig}
\usepackage{multicol}
\usepackage{multirow, cite}
\usepackage{amsfonts}
\usepackage{color}
\usepackage{subcaption}

\usepackage{mathtools}
\mathtoolsset{showonlyrefs}

\newtheorem{theorem}{Theorem}
\newtheorem{lemma}{Lemma}

\newtheorem{remark}{Remark}

\SetCommentSty{mycommfont}

\SetKwInput{KwInput}{Input}                
\SetKwInput{KwOutput}{Output}              
\SetKwInput{KwInitialize}{Initialize}
\SetKwComment{Comment}{/}{*/}

\title{Almost Cost-Free Communication in Federated Best Arm Identification
}
\author {
    Kota Srinivas Reddy,
    P. N. Karthik, and
    Vincent Y. F. Tan
}
\affiliations {
	National University of Singapore\\ Emails: ksvr1532@gmail.com, karthik@nus.edu.sg, vtan@nus.edu.sg
}
\begin{document}

\maketitle

\begin{abstract}
We study the problem of best arm identification in a federated learning multi-armed bandit setup with a central server and multiple clients. Each client is associated with a multi-armed bandit in which each arm yields {\em i.i.d.}\ rewards following a Gaussian distribution with an unknown mean and known variance. The set of arms is assumed to be the same at all the  clients. We define two notions of best arm---local and global. The local best arm at a client is the arm with the largest mean among the arms local to the client, whereas the global best arm is the arm with the largest  average mean across all the clients. We assume that each client can only observe the rewards from its local arms  and thereby estimate its local best arm. The clients communicate with a central server on  uplinks that entail a cost of $C\ge0$ units per usage per uplink. The global best arm is estimated at the server. The goal is to identify the local best arms and the global best arm with minimal total cost, defined as the sum of the total number of arm selections at all the clients and the total communication cost, subject to an upper bound on the error probability. We propose a novel algorithm {\sc FedElim}  that is based on successive elimination and communicates only in exponential time steps, and obtain a high probability  instance-dependent upper bound on its total cost. The key takeaway from our paper is that for any $C\geq 0$ and error probabilities sufficiently small, the total number of arm selections (resp.\ the total cost) under {\sc FedElim} is at most~$2$   (resp.~$3$) times the maximum total number of arm selections under its variant that communicates in every time step. Additionally, we show that the latter is optimal in expectation up to a constant factor, thereby demonstrating that communication is almost cost-free in {\sc FedElim}.
We numerically validate the efficacy of  {\sc FedElim} on two synthetic datasets and the MovieLens dataset.
 \end{abstract}

\section{Introduction}
\begin{figure}[!t]
\centering
\includegraphics[width=6.00cm]{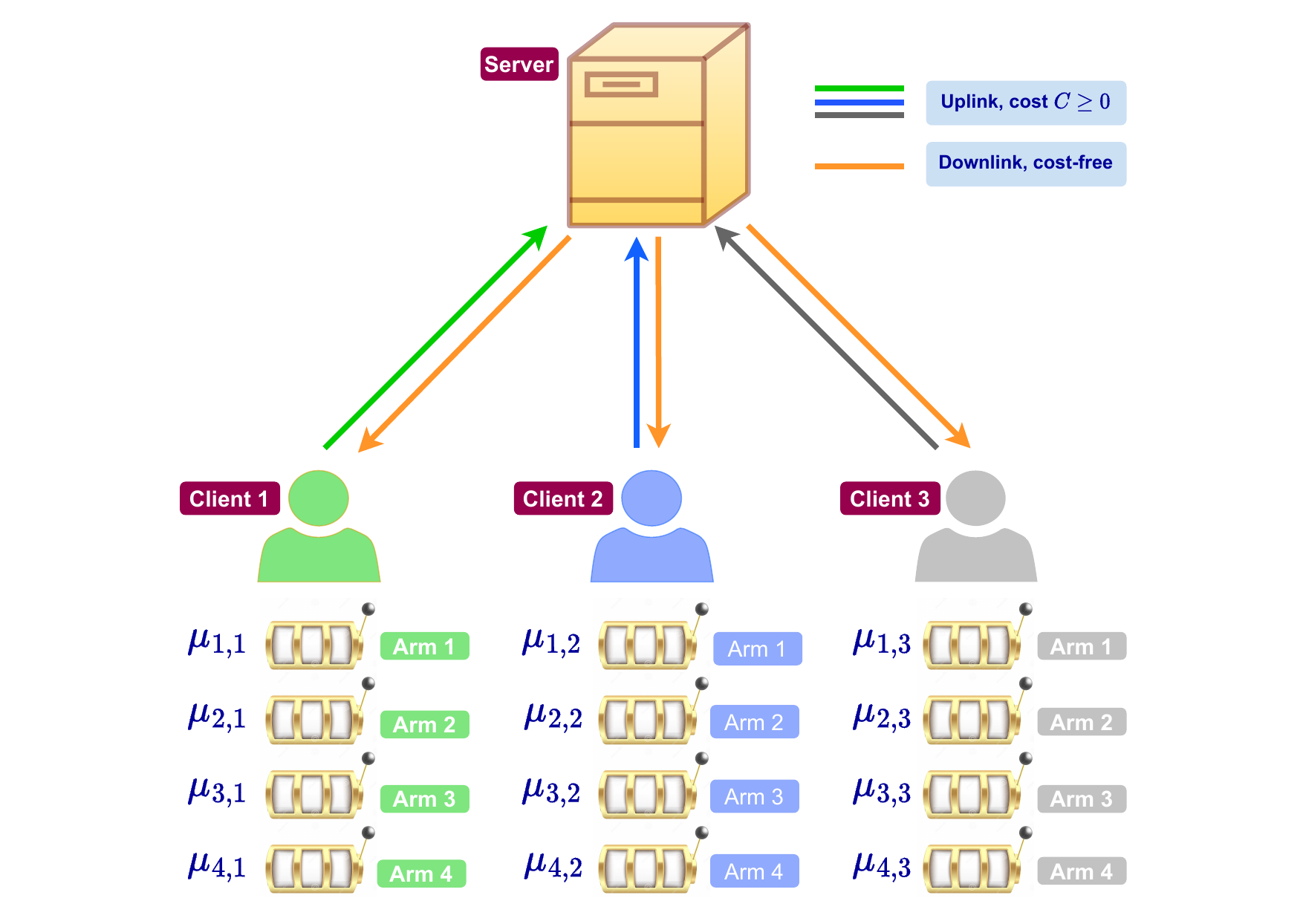} 
\caption{An illustration of our problem setup with $M=3$ clients and $K=4$ arms per client. The mean of arm $k$ of client $m$ is   $\mu_{k,m}$, where $k\in [K]$ and $m\in [M]$. Communication  of a scalar to the server is assumed to entail a cost of $C\geq 0$ units per usage of the uplink, whereas the downlink from the server to the client is  cost-free.} 
\label{fig:problem-setup}
\vspace{-.2in}
\end{figure}

We study an optimal stopping variant of the federated learning multi-armed bandit (FLMAB) regret minimisation problem of \citet{shi2021federated}. 
The specifics of our problem setup are as follows. We consider a federated multi-armed bandit setup with a central server and $M>1$ clients. Each client is associated with a multi-armed bandit with $K>1$ arms in which each arm yields independent and identically distributed ({\em i.i.d.}) rewards following a Gaussian distribution with an unknown mean and known variance. We assume that the set of arms is identical at all the clients. As in \citet{shi2021federated}, we consider two notions of best arm---{\em local} and {\em global}. The local best arm at a client is defined as the arm with the largest mean among the arms local to the client. The global best arm is the arm with the largest  average of the means averaged across the clients (we define these terms precisely later in the paper). We assume that each client can observe the rewards generated {\em only} from its local arms and thereby estimate its local best arm. The clients do not communicate directly with each other, but instead communicate with the central server. Communication from each client to the server entails a fixed cost of $C \ge 0$ units per usage per uplink. The information transmitted by the clients on the uplink is used by the server to estimate the global best arm. In contrast to the work of \citet{shi2021federated} where the goal is to minimise the regret accrued over a finite time horizon, the goal of our work is to find the local best arms of all the clients and the global best arm in a way so as to minimise the sum of the total number of arm pulls at the clients and the total communication cost, subject to an upper bound on the error probability. Figure \ref{fig:problem-setup} summarises our problem setup.


\subsection{Motivation}

The following example from the movie industry motivates our problem setup. Movies are typically categorised into various genres (for e.g., comedy, romance, action, thriller, etc.) and released in several parts (regions) of the world. The people of a region develop preferences for one or more genres courtesy of certain region-specific demographics (for e.g., age profile, females to males ratio of the population, etc.). Suppose that there are $M$ distinct regions and $K$ distinct genres. The following questions are  commonplace in the movie industry: (a) What genre of movies is most preferred {\em locally} by the people of a given region? (b) What genre of movies yields higher profits on the average {\em globally} across all regions? In the above example, a movie is akin to an arm and a region is akin to a client. The question in (a) above seeks to find the local best arm of each client, whereas the question in (b) seeks to find the global best arm. 

\subsection{Related Works}
Federated learning is an emerging paradigm that focuses on a distributed machine learning scenario in which there are multiple clients and a central server training a common machine learning model while keeping each client's local data private; see \citet{mcmahan2017communication} and \citet{kairouz2021advances} and the references therein for more details. 
The work of \citet{shi2021federated} extends the federated learning framework to multi-armed bandit paradigm and studies FLMAB under the theme of regret minimisation wherein 
the goal is to design arm selection algorithms to minimise the regret accrued over a finite time horizon. See \citet{lattimore2020bandit} and the references therein for more details on the regret minimisation theme and other related works on this theme. Contrary to the theme of regret minimisation,   {\em best arm identification}   falls under the theme of optimal stopping and can be embedded within the classical framework of \citet{chernoff1959sequential}. As noted in the work of \citet{bubeck2011pure} and \citet{zhong2021}, algorithms that are optimal in the context of regret minimisation are not necessarily so in the context of optimal stopping.

The problem of best arm identification is   well-studied and consists in finding the best arm (i.e., the arm with the largest mean value) in a (single) multi-armed bandit. This problem is studied under two complementary settings: (a) the {\em fixed-confidence setting}, where the objective is to find the best arm with the smallest expected number of arm pulls subject to ensuring that the error probability is no more than a given threshold value;  see \citet{even2006action, jamieson2014lil}, and (b) the {\em fixed-budget setting}, where the objective is to find the best arm as accurately as possible given a threshold on the number of arm pulls; see \citet{audibert2010best} and \citet{bubeck2011pure}. 
In this paper, we consider the fixed-confidence setting. For an excellent survey, see \citet{jamieson2014best}.

\citet{mitra2021exploiting}   study a federated variant of the best arm identification problem with a central server and multiple clients, similar to our work. However, their problem setting differs from ours in that in their work, the arms of a single multi-armed bandit are partitioned into as many subsets as there are clients. Each client is associated with a subset of arms and can observe only the rewards generated from the arms in this subset. The central goal in their paper is to identify the global best arm, defined as arm with the largest mean among the local best arms of the clients. Notice that an arm that is not the local best arm at any client cannot be the global best arm. Therefore, it suffices for each client to communicate to the server only the empirical mean of the estimated local best arm; this communication is assumed to take place periodically, only for time step $n\in \{1, H+1, 2H+1, \ldots\}$ for some fixed period $H$. However, in our work, the global best arm (defined  as the arm with the largest  average of the means averaged across the clients) may not necessarily be the local best arm at any client, because of which the clients may need to {\em communicate} the empirical means of their non-local best arms. Also, we propose an alternative strategy of communicating only at time steps $n=2^t$ for $t\in \mathbb{N}_0:=  \mathbb{N}\cup\{0\}$, and demonstrate that this strategy, called {\em exponentially sparse communication}, mitigates the overall communication cost and renders communication almost cost-free.
%
%
%


Works on collaborative learning in bandits (e.g., \citet{hillel2013distributed} and \citet{tao2019collaborative}) consider a central server and multiple clients as in our work, but with one salient difference: in the abovementioned works, the arms {\em and their distributions} are identical at all the clients (the goal is to establish collaboration among the clients to find the best arm faster than without collaboration). As a result, the local best arm of each client is identical to those of the other clients. In this paper, we assume that the set of arms is  identical at all the clients and allow for the arms to have {\em different} distributions across the clients, thereby leading to possibly distinct local and global best arms.

\subsection{Contributions}
We now highlight the key contributions of this paper.
\begin{itemize}
    \item We propose a novel algorithm   called {\sc FED}erated learning successive {\sc ELIM}ination algorithm (or {\sc FedElim}) for finding the local best arms and the global best arm (see Algorithm \ref{alg:FedElim}). The key feature of  {\sc FedElim} is that clients communicate to the server in only exponential time steps $n=2^t$ for some $t\in \mathbb{N}_0$. Given any $\delta\in (0,1)$, we show that  {\sc FedElim} declares the local best arms and the global best arm correctly with probability at least $1-\delta$. We present problem-instance dependent upper bounds on the total number of arm selections, the communication cost, and the total cost of {\sc FedElim}, each of which holds with probability at least $1-\delta$ (Theorem \ref{thm:FLSEA_withcost}). Our results show that the total cost of {\sc FedElim} scales as $\ln(1/\delta)$ in the error probability threshold $\delta$, and inversely as the squares of the {\em sub-optimality gaps} of the arms.
    
    \item For a variant of {\sc FedElim} (called $\textsc{FedElim}0$) that communicates in {\em every time step}, we obtain a high probability problem instance-dependent upper bound on the total number of arm selections (Theorem \ref{thm:FLSEA_upperbound}). We also obtain a  lower bound on the expected total number of arm selections required by {\em any} algorithm which outputs the correct answer with probability at least $1-\delta$ (Theorem \ref{thm:lb1_c0}), and show that the upper and the lower bounds are tight when $M$ is constant or when $M$ is sufficiently large. 
    
    \item The key takeaway from our paper is that for any $C\geq 0$ and sufficiently small $\delta$, the  total cost of {\sc FedElim} is at most  $3$ times  the  total number of arm selections under $\textsc{FedElim}0$ 
    with probability at least $1-\delta$. That is, communication is {\em almost cost-free} in {\sc FedElim}. 
    Through extensive simulations on two synthetic datasets and the large-scale, real-world MovieLens dataset, we compare the total cost of {\sc FedElim} with that of a periodic communication protocol with period $H$ based on successive elimination, and observe that there is a ``sweet spot'' for $H$ where the total cost of the latter is minimal. Determining this sweet spot requires knowing $C$ and other problem instance-specific constants and is infeasible in most practical settings. In comparison, {\sc FedElim}, while being agnostic to $C$ and other problem instance-specific constants, learns this sweet spot on-the-fly.

\end{itemize}

Although the focus of our paper is best arm identification, 
{\sc FedElim}   may be adapted to solve more general problems such as top-$N$ arms identification \citep{kalyanakrishnan2012pac}, thresholding in bandits \citep{locatelli2016optimal}, $\epsilon$-optimal arm identification \citep{even2006action}, and so on.
 In our paper, the Gaussian rewards assumption is merely for simplicity in the presentation. Our analyses are   applicable to observations that are sub-Gaussian. For more details, see Remarks \ref{rmk:subG} and \ref{rmk:sG_lb}.

\section{Notations and Problem Setup}\label{sec:notations_problemstatemt}

In this section, we lay down the notations used throughout the paper, and specify the problem setup. We consider a federated multi-armed bandit with a central server and $M$ clients. Each client is associated with a multi-armed bandit with $K$ arms (called {\em local} arms). We refer to the $K$-armed bandit associated with a client as its {\em local} multi-armed bandit. 
We write $[K] \coloneqq \{1,2,\ldots,K\}$ to denote the set of arms, and assume that $[K]$ is the same for all the clients and the server. Also, we write $[M]$ to denote the set of clients. 

\subsection{Local and Global Best Arms}

There are  $M$ local multi-armed bandits, one associated with each client.
For $n\in \mathbb{N}$,  let $X_{k,m}(n)$ denote the reward (or observation)  generated from local arm $k$ of client $m$ at time~$n$. For each $(k,m)$ pair, $\{X_{k,m}(n): n\geq 1\}$ is an  {\em i.i.d.} process following a Gaussian distribution with an unknown mean $\mu_{k,m}\in \mathbb{R}$ and known variance $\sigma^2$. For simplicity,  we set $\sigma^2=1$.
We define the {\em local best arm} $k_m^*$ of client $m$ as the arm with the largest mean among the local arms of client $m$, {\em i.e.,}  $k^*_m\coloneqq \arg\max_k \mu_{k,m}$; we assume that $k_m^*$ is unique for each $m$. Also,   let $\mu^{*}_m\coloneqq \mu_{k_m^*,m}=\max_k \mu_{k,m}$ be the mean of the local best arm of client $m$. Note that two different clients may have distinct local best arms. 
%
Letting $\mu_k:=\sum_{m=1}^{M}\mu_{k,m}/M$, we define the {\em global best arm} $k^*$ as the arm with the largest value of $\mu_k$, i.e., $k^*=\arg\max_k \mu_{k}$, and assume that $k^*$ is unique. We let $\mu^{*}\coloneqq \mu_{k^*}=\max_k \mu_{k}$ denote the mean of the global best arm. The global best arm may not necessarily be the local best arm at any   client. 
 
\subsection{Communication Model}
We assume that each client can observe only the rewards generated from its local arms, based on which the client can estimate its local best arm. Estimating the global best arm requires exchange of information among the clients. We assume that each client communicates with a central server, and that there is no direct communication between any two clients. We also assume that the communication link from a client to the server (uplink) entails a fixed cost of $C \ge 0$ units per usage of the link, and that the communication link from the server to the client (downlink) is cost-free as in \citet{hanna2022solving}. Each client sends real-valued information about the rewards from one or more of its local arms on its uplink. The server aggregates the information coming from all the clients to construct a set of arms that are potential contenders for being the global best arm, and communicates this set to each of the clients on the downlink. Each client selects each arm in set received from the server to obtain a more refined estimate of the arm's empirical mean. In this way, the clients and the server communicate until there is exactly one contender arm at the server.

When $C=0$, it is clearly advantageous for the clients to communicate with the server at every time step. When $C>0$, it is, however, beneficial for the clients to communicate with the server only  {\em intermittently} so that the   overall communication cost will be minimized. An instance of periodic communication in federated multi-armed bandits, where the clients communicate with the server periodically, once every $H$ time steps for a fixed integer $H>0$, may be seen in \citet{mitra2021exploiting}. An alternative communication strategy, one that we explore in this paper, is for the clients to communicate with the server only at time steps $n$ of the form $n=2^t$ for $t\in  \mathbb{N}_0$. As we shall see shortly, the latter strategy mitigates the communication costs significantly and renders communication almost cost-free. 

\subsection{Problem Instance and Algorithm}

A {\em problem instance} is identified by the matrix $\boldsymbol{\mu}=[\mu_{k,m}:k\in[K],m\in[M]]\in\mathbb{R}^{K\times M}$ of the means of the local arms of all the clients. The actual value of $\boldsymbol{\mu}$ is unknown, and the goal is to find the local best arm at each of the clients and also the global best arm (i.e., the vector $\mathbf{S}(\boldsymbol{\mu})\coloneqq (k_1^*,k_2^*,\ldots, k_{M}^*,k^*)\in [K]^{m+1}$) with high confidence. Each client selects one or more of its local arms at every time  $n \in  \mathbb{N}$ and forms an estimate of its local best arm as the arm with the largest empirical mean at time step~$n$.


An {\em algorithm} for finding the local best arms and the global best arm prescribes the following:
\begin{itemize}
	\item A {\em selection rule} that specifies the arm(s) that each client must select from amongst its local arms for each $n$.
	\item A {\em communication rule} that specifies the condition(s) under which the clients will communicate with the server 
and the information that the clients will send to the server.
	\item A {\em termination rule} that specifies when to stop further selection of arms at the clients.
	\item A {\em declaration rule} that specifies the estimates $ \hat{\mathbf{S}} \coloneqq (\hat{k}^*_1,\hat{k}^*_2,\ldots,\hat{k}^*_M,\hat{k}^*)\in [K]^{M+1}$ of the local best arms and the global best arm to output; here, $\hat{k}^*_m$ is the estimate of the local best arm of client $m \in [M]$ and $\hat{k}^*$ is the estimate of the global best arm.
\end{itemize}
We denote an  algorithm by $\pi$ and 
define its
{\em total cost} 
\begin{align}
	C^{\mathrm{total}}(\pi) &= (\text{total number of arm pulls under }\pi \nonumber\\*
	&\quad+ \text{total communication cost under } \pi).
	\label{eq:cost-of-an-algorithm}
\end{align}   
In \eqref{eq:cost-of-an-algorithm}, the first component on the right hand side represents the total number of arm selections made by all the clients until termination, and the second component is the total communication cost incurred across all the clients. 

\subsection{Objective}
For $\delta\in (0,1)$, an algorithm $\pi$ is said to be {\em $\delta$-probably approximately correct} (or $\delta$-PAC) if for all $\boldsymbol{\mu}\in \mathbb{R}^{K\times M}$, we have $P_{\boldsymbol{\mu}}^{\pi}(\hat{\mathbf{S}}=\mathbf{S}(\boldsymbol{\mu})) \geq 1-\delta$; here, $P_{\boldsymbol{\mu}}^\pi(\cdot)$ denotes the probability measure under algorithm $\pi$ and   problem instance $\boldsymbol{\mu}$. Note that any $\delta$-PAC algorithm $\pi$ must declare the correct output with probability at least $1-\delta$ for all problem instances $\boldsymbol{\mu}$, as $\pi$ is oblivious to the knowledge of the underlying problem instance. Given any $\boldsymbol{\mu}$ and $\delta\in (0,1)$, our objective is to design a $\delta$-PAC algorithm, say $\pi^*$, for finding the local best arms and the global best arm, and derive a $\boldsymbol{\mu}$-dependent upper bound, say $U(\boldsymbol{\mu}, \delta)$, on its total cost $C^{\mathrm{total}}(\pi^*)$, such that
\begin{equation}
	P_{\boldsymbol{\mu}}^{\pi^*}\big(C^{\mathrm{total}}(\pi^*) \leq U(\boldsymbol{\mu}, \delta)\big) \geq 1-\delta.
\end{equation}


In the following section, we present a version of the well-known successive elimination algorithm of  \citet{even2006action} for finding the local best arms and the global best arm. We interleave it with the exponentially sparse communication sub-protocol, and subsequently obtain a high probability upper bound on its total cost.

\section{The Federated Learning Successive Elimination Algorithm ({\sc FedElim})}
Our algorithm, termed {\em Federated Learning Successive Elimination Algorithm} (or {\sc FedElim}), is presented in Algorithm~\ref{alg:FedElim}. In the following, we provide some algorithm-specific notations and a detailed description of {\sc FedElim}.

\subsection{Algorithm-Specific Notations}
\label{subsec:algorithm-notations}
The {\sc FedElim} algorithm proceeds in several time steps; we denote a generic time step by $n\in \mathbb{N}$. An arm is said to be a {\em local active arm} of client $m$ if it is still a contender for being the client's local best arm. On the other hand, an arm is said to be a {\em global active arm} at the central server if it is still a contender for being the global best arm. At any given time step, we let $\mathcal{S}_{\mathrm{l},m}$ and $\mathcal{S}_{\mathrm{g}}$ denote respectively the set of local active arms at client $m$ and the set of global active arms at the server. 
We write $\hat{\mu}_{k,m}(n)$ to denote the empirical mean of arm $k$ of client $m$ at time step $n$, and define $\hat{\mu}_{k}(n)\coloneqq\sum_{m=1}^{M}\hat{\mu}_{k,m}(n)/M$. We let $\alpha_{\mathrm{l}}(n)\coloneqq \sqrt{\frac{2\ln{(8KMn^2/\delta)}}{n}}$ and $\alpha_{\mathrm{g}}(n) \coloneqq \sqrt{\frac{2\ln{(8Kn^2/\delta)}}{Mn}}$ denote respectively the {\em local confidence parameter} and the {\em global confidence parameter} in time step $n$.

\subsection{Algorithm Description}
{\em At each client:} In each time step $n$, the algorithm first computes $\mathcal{S}_m=\mathcal{S}_{\mathrm{l},m}\cup \mathcal{S}_{\mathrm{g}}$ for each $m\in [M]$. If $|\mathcal{S}_m|>1$, the algorithm selects each arm in $\mathcal{S}_m$ once and updates their respective empirical means ({\em selection rule}).
Next, for each $m\in [M]$, the algorithm checks for the validity of the condition $|\mathcal{S}_{\mathrm{l},m}|>1$. If this condition holds, the algorithm eliminates all those arms in $\mathcal{S}_{\mathrm{l},m}$ that are no more contenders for being the local best arm of client $m$. This is accomplished as follows: for each $m\in [M]$, the algorithm computes $\hat{\mu}_{*,m}(n) \coloneqq \max_{k\in \mathcal{S}_{\mathrm{l},m}} \mu_{k,m}(n)$, and eliminates arm $k$ from $\mathcal{S}_{\mathrm{l},m}$ if
    $\hat{\mu}_{*,m}(n)-\hat{\mu}_{k,m}(n)>2\alpha_{\mathrm{l}}(n)$.
The arms remaining in $\mathcal{S}_{\mathrm{l},m}$ after elimination are the local active arms of client $m$. For each $m\in [M]$, if $|\mathcal{S}_{\mathrm{l},m}|=1$ after elimination, the algorithm outputs the arm in $\mathcal{S}_{\mathrm{l},m}$ as the local best arm of client $m$ ({\em declaration rule} for client $m$). 

{\em At the server:} After working on $\mathcal{S}_{\mathrm{l},m}$ for each $m\in [M]$ as outlined above, the algorithm checks if $|\mathcal{S}_{\mathrm{g}}|>1$ and if $n=2^t$ for some $t\in \mathbb{N}_0$. If both of these conditions hold, then each client $m\in [M]$ sends to the server its estimates $\{\hat{\mu}_{k,m}(n): k\in \mathcal{S}_{\mathrm{g}}\}$ of the empirical means of the arms in $\mathcal{S}_{\mathrm{g}}$, one per usage of its uplink ({\em communication rule}). Because the uplink entails a cost of $C\geq 0$, the communication cost incurred at a client is $C\,|\mathcal{S}_{\mathrm{g}}|$, and therefore the total communication cost across all the clients is $C\,M\,|\mathcal{S}_{\mathrm{g}}|$. The server eliminates all those arms in $\mathcal{S}_{\mathrm{g}}$ that are no more contenders for being the global best arm as follows: the server first computes $\hat{\mu}_k(n) = \sum_{m=1}^{M}\hat{\mu}_{k,m}(n)/M$ for each $k\in \mathcal{S}_{\mathrm{g}}$ and also $\hat{\mu}_{*}(n) = \max_{k\in \mathcal{S}_{\mathrm{g}}} \hat{\mu}_k(n)$, and eliminates arm $k$ from $\mathcal{S}_{\mathrm{g}}$ if $\hat{\mu}_{*}(n) -  \hat{\mu}_k(n) > 2\alpha_{\mathrm{g}}(n)$. 
The arms remaining in $\mathcal{S}_{\mathrm{g}}$ after elimination are the global active arms. If $|\mathcal{S}_{\mathrm{g}}|=1$ after elimination, the algorithm outputs the arm in $\mathcal{S}_{\mathrm{g}}$ as the global best arm ({\em declaration rule} for the global best arm). 

Upon identifying the local best arms and the global best arm, the algorithm {\em terminates}. Else, if at least one of the local best arms or the global best arm is not identified, the algorithm continues to the next time step.

\begin{remark}
Recall that in our problem setup, the global best arm may not necessarily be the local best arm at any client. In fact, the local best arms and the global best arms can be all distinct. As a result, even if an arm (say arm $k$) is eliminated from $\mathcal{S}_{\mathrm{l},m}$ at client $m$ (i.e., arm $k$ is not the local best arm of client $m$), it may still need to be selected further before it can be eliminated globally from $\mathcal{S}_{\mathrm{g}}$, and vice-versa. It is for this reason that we set $\mathcal{S}_m=\mathcal{S}_{\mathrm{l},m}\cup \mathcal{S}_{\mathrm{g}}$ as the set of arms to be selected at client $m$. In contrast, when the global best arm is always one of the local best arms, as in \citet{mitra2021exploiting}, eliminating an arm locally at a client is akin to eliminating the arm globally.    
\end{remark}


\begin{algorithm}[!t]
\DontPrintSemicolon
  
  \KwInput{$K \in \mathbb{N}$, $M\in \mathbb{N}$, $\delta\in (0,1)$}
  \KwOutput{$(\hat{k}_1^*, \ldots, k_M^*, \hat{k}^*)\in [K]^{M+1}$}
  \KwInitialize{$n=0$, $\hat{\mu}_{k,m}(n)=0$ and $\mathcal{S}_{\mathrm{l},m}=[K]$ for all $k,m$, $\hat{\mu}_{k}(n)=0$ and $\mathcal{S}_{\mathrm{g}}=[K]$ for all~$k$,  $\texttt{run}=\texttt{true}$}
  \While{{\em $\texttt{run}=\texttt{true}$}}
   {
   {$n \leftarrow n+1$}
   
   		\For{$m=1:M$}    
        { 
        
        
        {$\mathcal{S}_m \leftarrow \mathcal{S}_{\mathrm{l},m}\cup \mathcal{S}_{\mathrm{g}}$} \tcp*[f]{\tiny Arms client $m$ selects}

%
	\If(\tcp*[f]{\tiny Selection rule}){$|\mathcal{S}_{m}|>1$}{\For{$k\in \mathcal{S}_m$}{ pull arm $k$ of client $m$ and update its empirical mean $\hat{\mu}_{k,m}(n)$}} 
        
        \If{$|\mathcal{S}_{\mathrm{l},m}|>1$  }
        { 
        Set $\hat{\mu}_{*,m}(n)=\max_{k \in \mathcal{S}_{\mathrm{l},m}}\hat{\mu}_{k,m}(n)$ 
         
         \For(\tcp*[f]{\tiny Inactive local arms elimination}){$k\in \mathcal{S}_{\mathrm{l},m}$ such that $\hat{\mu}_{*,m}(n)-\hat{\mu}_{k,m}(n)\geq 2\alpha_{\mathrm{l}}(n)$}    
        { $\mathcal{S}_{\mathrm{l},m} \leftarrow \mathcal{S}_{\mathrm{l},m}\backslash \{k\}$}
         }
         
         \If(\tcp*[f]{\tiny Declaration rule}){$|\mathcal{S}_{\mathrm{l},m}|=1$}{
         Output $ \hat{k}_m^*\in\mathcal{S}_{\mathrm{l},m}$ 
         
         $\mathcal{S}_{\mathrm{l},m}\leftarrow \emptyset$}
      } 
      
      \If(\tcp*[f]{\tiny Communication rule}){$|\mathcal{S}_{\mathrm{g}}|>1$ and $n=2^t$ for some $t\in  \mathbb{N}_0$}{
       \For{$k\in \mathcal{S}_{\mathrm{g}}$}    
        {For each $m\in [M]$, client $m$ sends $\hat{\mu}_{k,m}(n)$ to the server.
        
         Set $\hat{\mu}_{k}(n) = \sum_{m=1}^{M}\hat{\mu}_{k,m}(n)/M$} 
         
        Set $\hat{\mu}_{*}(n)=\max_{k \in \mathcal{S}_{\mathrm{g}}}\hat{\mu}_{k}(n)$ 
         
         \For(\tcp*[f]{\tiny Inactive global arms elimination}){$k\in \mathcal{S}_{\mathrm{g}}$ such that $\hat{\mu}_{*}(n)-\hat{\mu}_{k}(n)\geq 2\alpha_{\mathrm{g}}(n)$}    
        { $\mathcal{S}_{\mathrm{g}} \leftarrow \mathcal{S}_{\mathrm{g}}\backslash \{k\}$}
         }
         
         \If(\tcp*[f]{\tiny Declaration rule}){$|\mathcal{S}_{\mathrm{g}}|=1$}{
        Output $\hat{k}^*\in \mathcal{S}_{\mathrm{g}}$ 
         
         $\mathcal{S}_{\mathrm{g}}\leftarrow \emptyset$}

         \If(\tcp*[f]{\tiny Termination rule}){$|\mathcal{S}_{m}|=0$ for all  $m\in[M]$}{  $\texttt{run}=\texttt{false}$}
   }
\caption{Federated Learning Successive Elimination Algorithm ({\sc FedElim})}
\label{alg:FedElim}
\end{algorithm}

\begin{remark}\label{rmk:knowledge_of_C}
To keep the total cost of an algorithm small, it is imperative to strike a  balance between the total number of arm selections and the communication cost. For instance, as is naturally expected and also demonstrated by our numerical results later in the paper, a periodic communication scheme with period $H$ and based on successive elimination incurs a larger communication cost than our exponentially sparse communication scheme (see Figures~\ref{subfig:comm-cost-comparison} and~\ref{subfig:comm-cost-comparison-movielens}). With regard to the total number of arm selections, one might expect that the periodic communication protocol outperforms our exponential sparse communication protocol because more frequent communication in the former leads to faster identification of the global best arm. From our numerical results, we find that this is true only partially. Rather interestingly, Figures~\ref{subfig:total-cost-comparison} and~\ref{subfig:total-cost-comparison-movielens} indicate that the total cost of a periodic communication algorithm (based on successive elimination) with period $H$ decreases, attains a minimum, and thereafter increases with increase in $H$, thereby suggesting that there is ``sweet spot'' for $H$, say $H_{\mathrm{opt}}$, where the total cost is minimal. However, $H_{\mathrm{opt}}$ is, in general, a function of $C$ and problem instance-specific constants which are not known beforehand in most practical settings, thereby making the computation of $H_{\mathrm{opt}}$ infeasible. Figures~\ref{subfig:total-cost-comparison} and~\ref{subfig:total-cost-comparison-movielens} show that {\sc FedElim} finds this sweet spot while being agnostic to $C$ and other problem instance-specific constants, and thereby achieves a balanced trade-off between the total number of arm selections and comm. cost.
\end{remark}

\section{Performance Analysis of {\sc FedElim}}\label{sec:performance}
In this section, we present our theoretical results on the performance (total number of arm pulls, total communication cost, and the total cost) of {\sc FedElim}. We only state the results below and provide the proofs in the supplementary material. The first result below asserts that given any $\delta \in (0,1)$, {\sc FedElim} is $\delta$-PAC, i.e., it identifies the local best arms and the global best arm  correctly with probability at least $1-\delta$.

\begin{theorem}\label{thm:correctness}
Given any $\delta \in (0,1)$, {\sc FedElim} identifies the local best arms and global best arm correctly with probability at least $1-\delta$ and is thus $\delta$-PAC.
\end{theorem}
In the proof (presented in the supplementary material), we first show that for any $\delta \in (0,1)$, the event 
\begin{align}
	\!\!\mathcal{E}\! \coloneqq\! \bigcap_{n\in\mathbb{N}, k\in [K] , m\in [M]} \left\{\!\!\!  \begin{array}{c} \hspace{.32in} |\hat{\mu}_k(n)- \mu_k| \!\leq\! \alpha_{\mathrm{g}}(n),\! \\ |\hat{\mu}_{k,m}(n) - \mu_{k,m}| \!\leq \!\alpha_{\mathrm{l}}(n)  \end{array} \!\!\right\} 
	\label{eq:event-E}
\end{align}
has probability at least $1-\delta$; this is established using a standard inequality on the concentration of the empirical mean around the true mean for Gaussian rewards. We then show that {\sc FedElim} always outputs the correct answer under $\mathcal{E}$.

We now analyse a variant of {\sc FedElim} called $\textsc{FedElim}0$ which communicates to the server {\em in every time step}. Specifically, $\textsc{FedElim}0$ differs from {\sc FedElim} in line $15$ of Algorithm~\ref{alg:FedElim}, which is executed for all $n$ in $\textsc{FedElim}0$ but only for $n=2^t$ for $t\in \mathbb{N}_0$ in {\sc FedElim}. Our interest is only in the total number of arm selections of $\textsc{FedElim}0$, say $T_{\textsc{FedElim}0}$, required to find the local best arms and the global best arm on the event  $\mathcal{E}$, and how this compares with the total number of arm selections of other algorithms which also communicate in every time step. As we shall   see, $T_{\textsc{FedElim}0}$ is an important term that governs the problem instance-dependent upper bounds for the total number of arm selections and the total cost of {\sc FedElim}. Note that $T_{\textsc{FedElim}0}$ is also the total cost of $\textsc{FedElim}0$ on $\mathcal{E}$ when $C=0$.

\subsection{Performance Analysis of $\textsc{FedElim}0$} 
%
%

For $k\neq k^*_m$,   let $\Delta_{k,m}\coloneqq\mu_{k^*_m,m}-\mu_{k,m}$ denote the {\em suboptimality gap} between the means of arm $k$ of client $m$ and the local best arm of client $m$, and let $\Delta_{k^*_m,m}\coloneqq\min_{k\neq k^*_m}\Delta_{k,m}$.  Similarly, for $k\neq k^*$, we let $\Delta_{k}\coloneqq \mu_{k^*}-\mu_{k}$ and $\Delta_{k^*}\coloneqq \min_{k\neq k^*}\Delta_{k}$. The following result provides a problem instance-dependent upper bound on $T_{\textsc{FedElim}0}$.

 
 \begin{theorem} \label{thm:FLSEA_upperbound}
 Fix $\delta \in (0,1)$. On the event $\mathcal{E}$ defined in \eqref{eq:event-E}, 
 \begin{align}
 T_{\textsc{FedElim}0} &\leq T \coloneqq \sum_{k=1}^{K}\ \sum_{m=1}^{M}\max \big\{ T_{k,m}, \ T_k \big\},
 	\label{eq:high-prob-upper-bound-c-is-zero}
 \end{align}
  where for each $k\in[K]$ and $m\in [M]$,
 \begin{align}
 T_{k,m} & \coloneqq  102\cdot  \frac{\ln\Big(\frac{64\sqrt{\frac{8KM}{\delta}}}{\Delta_{k,m}^2}\Big)}{\Delta_{k,m}^2} + 1,\label{eq:T_km}\\
 T_k &\coloneqq 102\cdot \frac{\ln\Big(\frac{64\sqrt{\frac{8K}{\delta}}}{M\Delta_{k}^2}\Big)}{M\Delta_{k}^2} + 1. \label{eq:T_k}
 \end{align}
 \end{theorem} 
We show in the proof that on the event $\mathcal{E}$, the total number of arm selections under $\textsc{FedElim}0$ required to identify arm $k$ of client $m$ as the client's local best arm or otherwise, say $T_{k,m}^{(1)}$, is upper bounded by $T_{k,m}$ for all $k\in [K]$ and $m\in [M]$. To establish the preceding result, we use the fact that $\alpha_{\mathrm{l}}(n) \to 0$ as $n\to\infty$, and look for the smallest integer $n$ such that $\alpha_{\mathrm{l}}(n) \leq \Delta_{k,m}^2/4$; call this $n_{k,m}$. We argue that $T_{k,m}^{(1)} \leq n_{k,m}$ on the event $\mathcal{E}$, and subsequently show that $T_{k,m}$ is an upper bound for $n_{k,m}$. A similar procedure as above is used to upper bound the total number of arm pulls required to identify arm $k$ as being the global best arm or otherwise at the server. Combining the two upper bounds and noting that the event $\mathcal{E}$ occurs with probability at least $1-\delta$, we arrive at \eqref{eq:high-prob-upper-bound-c-is-zero}.

The next result shows that the upper bound in \eqref{eq:high-prob-upper-bound-c-is-zero} is tight up to a constant factor.

\begin{theorem}\label{thm:lb1_c0}
Given $\delta \in (0,1)$ and a $\delta$-PAC algorithm $\pi$, let $T_\delta^\pi$ denote the total number of arm selections under $\pi$ required to find the local best arms and the global best arm when the clients and the server communicate in every time step. Under the problem instance $\boldsymbol{\mu}$, 
\begin{equation}
	\inf_{\pi\text{ is }\delta\text{-PAC}}\ \mathbb{E}_{\boldsymbol{\mu}}^\pi [T_{\delta}^{\pi}] \geq \sum_{k=1}^{K}\sum_{m=1}^{M}\max\left\lbrace \frac{\ln(\frac{1}{2.4\delta})}{\Delta_{k,m}^2},\frac{\ln(\frac{1}{2.4\delta})}{M^2\Delta_{k}^2} \right \rbrace,
	\label{eq:lower-bound-1}
\end{equation}
where in \eqref{eq:lower-bound-1}, $\mathbb{E}_{\boldsymbol{\mu}}^\pi[\cdot]$ denotes the expectation under the algorithm $\pi$ and the problem instance $\boldsymbol{\mu}$.
\end{theorem}
The proof of Theorem \ref{thm:lb1_c0} is based on the   transportation lemma \cite[Lemma~1]{kaufmann2016complexity} which combines a certain change of measure technique and Wald's identity for {\em i.i.d.} processes.

\begin{remark}
Theorems \ref{thm:FLSEA_upperbound} and \ref{thm:lb1_c0} together provide a fairly tight characterisation of the total number of arm selections under the optimal algorithm in the class of all algorithms that communicate in every time step. They  show that $\textsc{FedElim}0$ is almost optimal in this class. Neglecting the logarithm terms and the constants, the key difference between the upper and lower bounds manifests in the second term in the maximum in \eqref{eq:lower-bound-1}, in which there is an additional factor of $M$ in the denominator. When $M$ is a constant or if $M$ is so large so that $\Delta_{k,m}\le\sqrt{M}\ \Delta_k$ for all $(k,m) \in [K]\times [M]$ (a typical federated learning scenario in which the number of clients  $M$ is large),  Theorems~\ref{thm:FLSEA_upperbound} and~\ref{thm:lb1_c0} are tight up to log factors.  
\end{remark}

 \subsection{Performance of {\sc FedElim} with Uplink Cost}
 
We now present a high-probability upper bound on the total cost (i.e., the sum of the total number of arm pulls and the total communication cost) of {\sc FedElim} for any $C \ge 0$. Given a problem instance $\boldsymbol{\mu}$, for each $k\in [K]$ and $m\in [M]$, let $T_{k,m}$ and $T_k$ be as defined in \eqref{eq:T_km} and \eqref{eq:T_k} respectively. 
 
 
\begin{theorem}\label{thm:FLSEA_withcost}
Fix a problem instance $\boldsymbol{\mu}$, uplink cost $C \ge 0$, and $\delta \in (0,1)$ such that $C \ln T_k \le T_k$ for all $k\in [K]$. 
Let $T_{\textsc{FedElim}}^C$, $C_{\textsc{FedElim}}^{\mathrm{comm}}$, and $C_{\textsc{FedElim}}^{\mathrm{total}}$ denote respectively the total number of arm selections, the communication cost, and the total cost of {\sc FedElim} towards identifying the local best arms and the global best arm. On the event $\mathcal{E}$ defined in~\eqref{eq:event-E}, the following  inequalities hold (with $T$ as defined in \eqref{eq:high-prob-upper-bound-c-is-zero}):
\begin{align}
	T_{\textsc{FedElim}}^C &\leq \sum_{k=1}^{K}\ \sum_{m=1}^{M} \max\{T_{k,m},\ 2 \,T_k\} \leq 2\,T, \label{eqn:first_line}\\
	C_{\textsc{FedElim}}^{\mathrm{comm}} & \leq C\cdot M \cdot \sum_{k=1}^{K} \left\lceil \frac{\ln T_k }{\ln 2 } \right\rceil,\label{eq:comm-cost-under-FedElim}\\
	C_{\textsc{FedElim}}^{\mathrm{total}} & = T_{\textsc{FedElim}}^C+C_{\textsc{FedElim}}^{\mathrm{comm}} \leq 3 \, T. \label{eq:total-cost-FedElim}
\end{align}
 \end{theorem}
 Notice that the maxima in \eqref{eq:high-prob-upper-bound-c-is-zero} and that in \eqref{eqn:first_line} are identical up to the constant~$2$. Intuitively, the extra factor of~$2$ arises because if a candidate arm $k$ is not eliminated in time step $n=2^t$ but is eliminated in time step $n=2^{t+1}$ for some $t \in \mathbb{N}_0$, then it must be the case that $2^{t+1} \le  2\,T_k$, and therefore the total number of arm selections is at most $2\,T_k$.  
 
It is no coincidence that the constant $2$ appears inside the maximum in~\eqref{eqn:first_line} and also in the denominator in \eqref{eq:comm-cost-under-FedElim}. In fact, exponential sparse communication in time steps $n=\lceil\lambda^t\rceil$ for $t\in \mathbb{N}_0$ and $\lambda>0$, results in $\lambda$ replacing $2$ in  both \eqref{eqn:first_line} and \eqref{eq:comm-cost-under-FedElim}. Then, optimising the sum of the $\lambda$-analogues of the right hand sides of \eqref{eqn:first_line} and \eqref{eq:comm-cost-under-FedElim}, we may arrive at a fairly tight upper bound on the total cost, i.e., the $\lambda$-analogue of~\eqref{eq:total-cost-FedElim}. However, the optimal $\lambda$ is, in general, a function of $C$ and the problem instance-specific sub-optimality gaps which are {\em unknown} in most practical settings. Therefore, we do away with finding the optimal $\lambda$ and instead use $\lambda=2$. For a more detailed discussion, see the supplementary material.

 \begin{remark}\label{rmk:independent_C}
      The key takeaway result of our paper, presented in inequality~\eqref{eq:total-cost-FedElim}, shows that the total number of arm selections (resp.\ total cost) of {\sc FedElim} is at most $2$ (resp.~$3$) times $T$. These multiplicative gaps of $2$ and $3$ do not depend on~$C$. In contrast, for periodic communication \citep{mitra2021exploiting} with period $H$, it can be shown that the multiplicative gap for the total cost is $1+C/H$, which does depend on the per usage communication cost $C$. 
 \end{remark}

\section{Numerical Results}
\label{sec:experiments}
\begin{figure*}[!t]
     \centering
     \begin{subfigure}[b]{1.275\textwidth/4}
         \centering
         \includegraphics[width=0.87\textwidth]{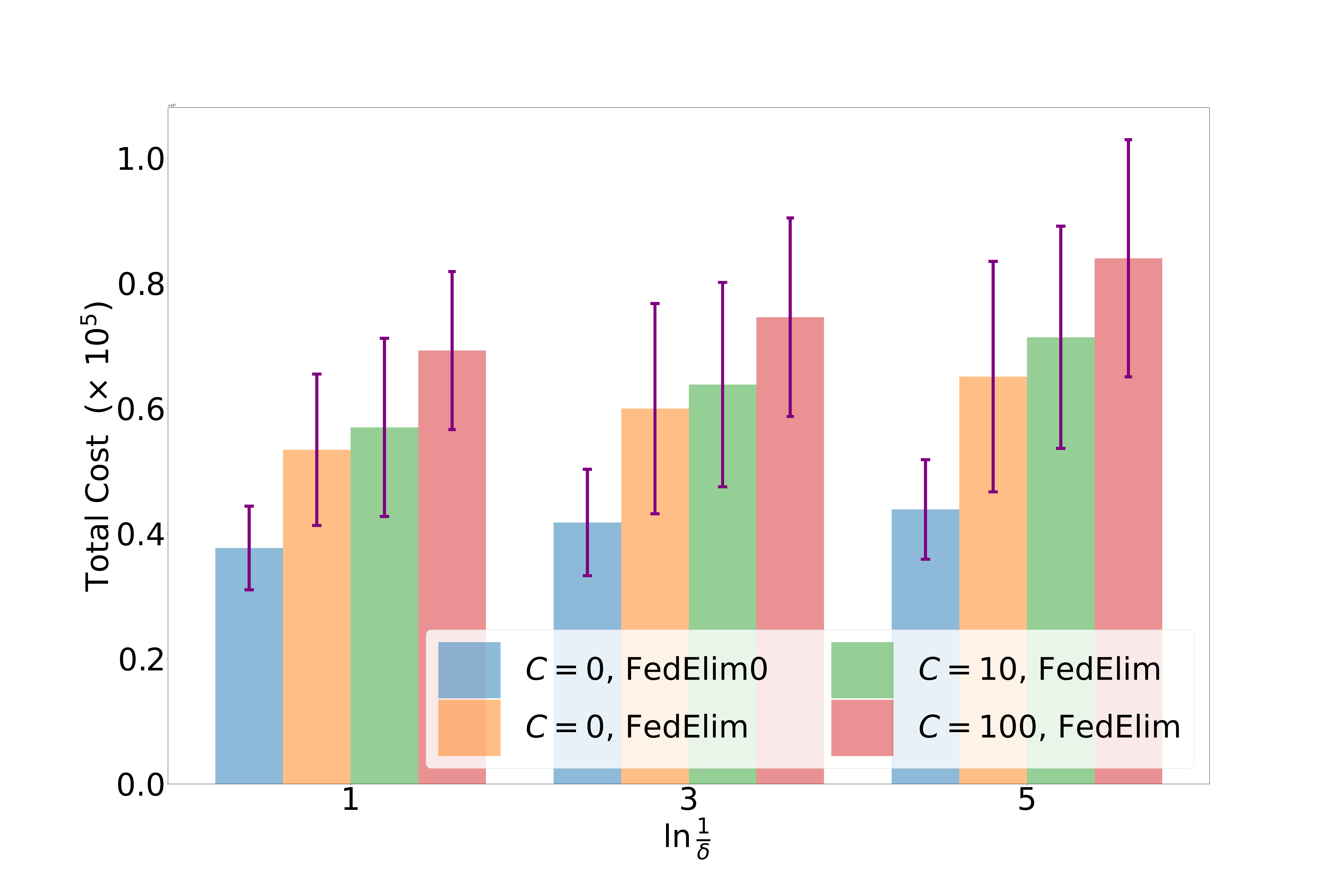}
         \caption{A plot comparing the  number of arm pulls of $\textsc{FedElim}0$ with the  cost of {\sc FedElim} for $C\in \{0, 10, 100\}$. $\textsc{FedElim}0$ has a lower cost compared to {\sc FedElim} when $C=0$.}
         \label{subfig:fedelim-total-cost-and-arm-pulls}
     \end{subfigure}
     \hspace{.05in}
     \begin{subfigure}[b]{1.275\textwidth/4}
         \centering
         \includegraphics[width=0.87\textwidth]{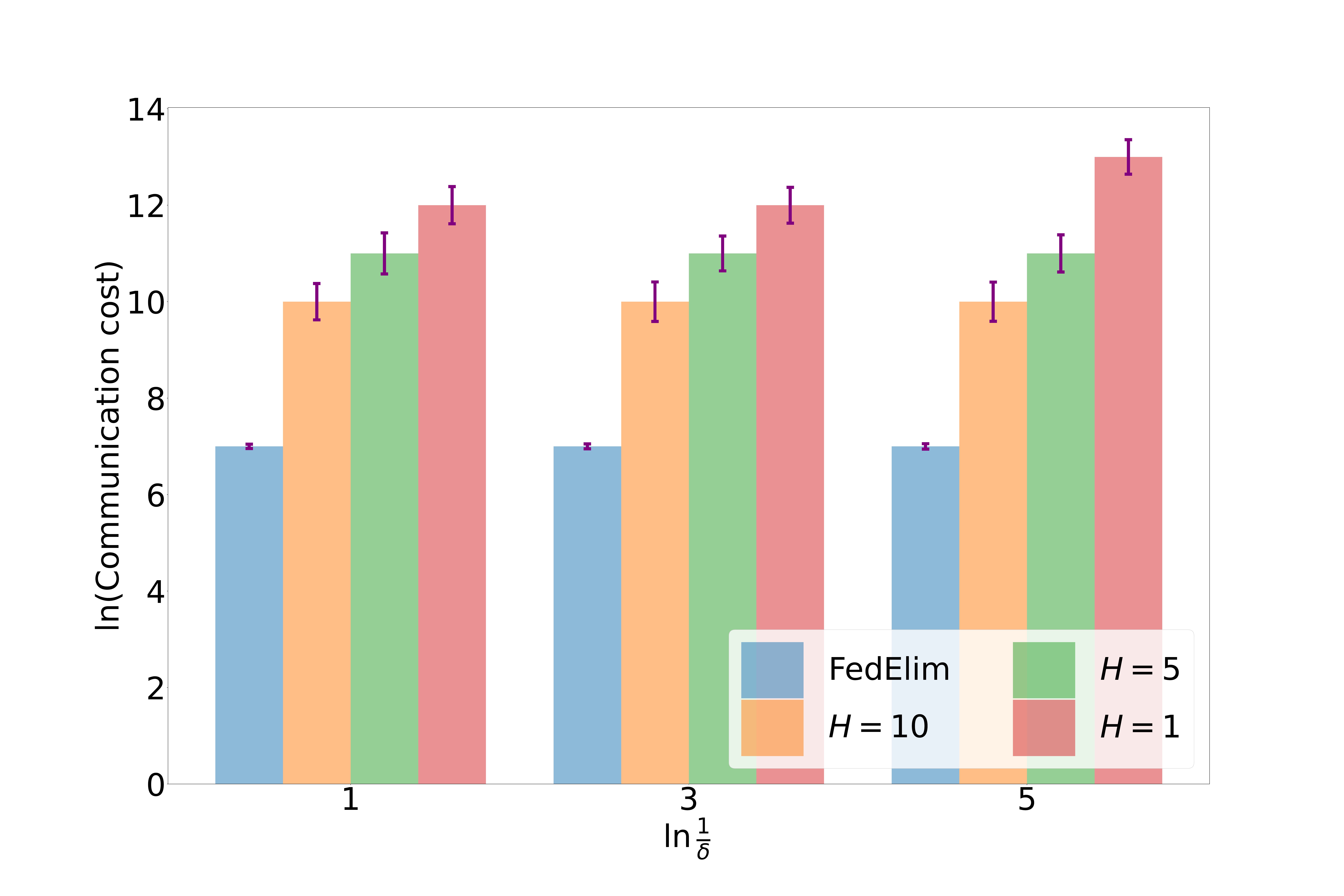}
         \caption{A   plot comparing of the communication cost  incurred under the proposed {\sc FedElim} algorithm and under periodic communication with period \mbox{$H \in \{1,5,10\}$.}}
         \label{subfig:comm-cost-comparison}
     \end{subfigure}
     \hspace{.05in}
     \begin{subfigure}[b]{1.275\textwidth/4}
         \centering
         \includegraphics[width=0.87\textwidth]{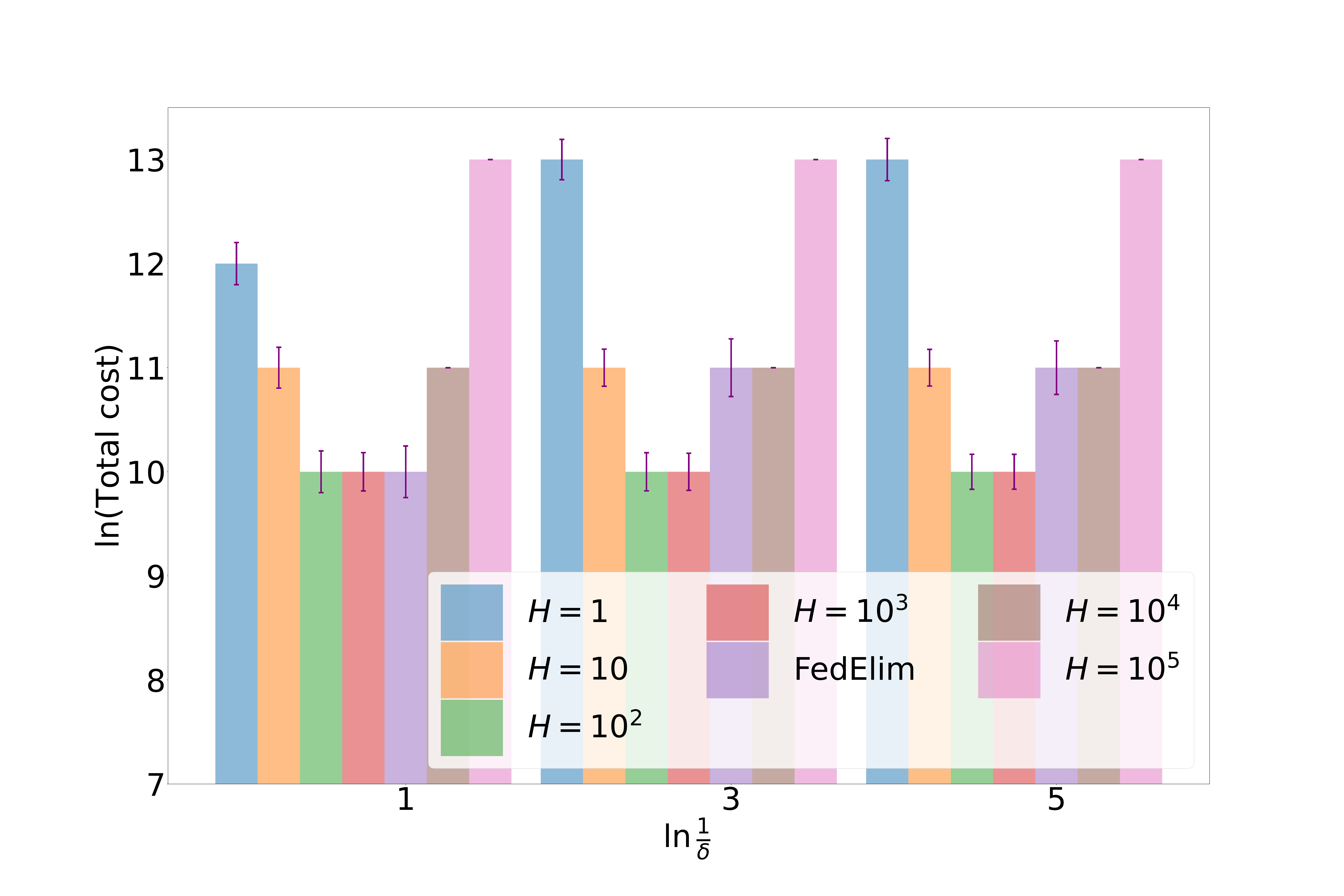}
         \caption{A   plot  comparing  the     cost     under  {\sc FedElim}   and    periodic communication with period $H =10^p$ for $p\in \{0,\ldots,5\}$.   {\sc FedElim} almost attains the minimum w/o knowing~$C$. }
         \label{subfig:total-cost-comparison}
     \end{subfigure}
        \caption{Numerical results on the synthetic dataset with the problem instance in \eqref{eqn:syn}.}
        \label{fig:three-figures}
\end{figure*} 

\begin{figure*}[!t]
     \centering
     \begin{subfigure}[b]{1.275\textwidth/4}
         \centering
         \includegraphics[width=0.87\textwidth]{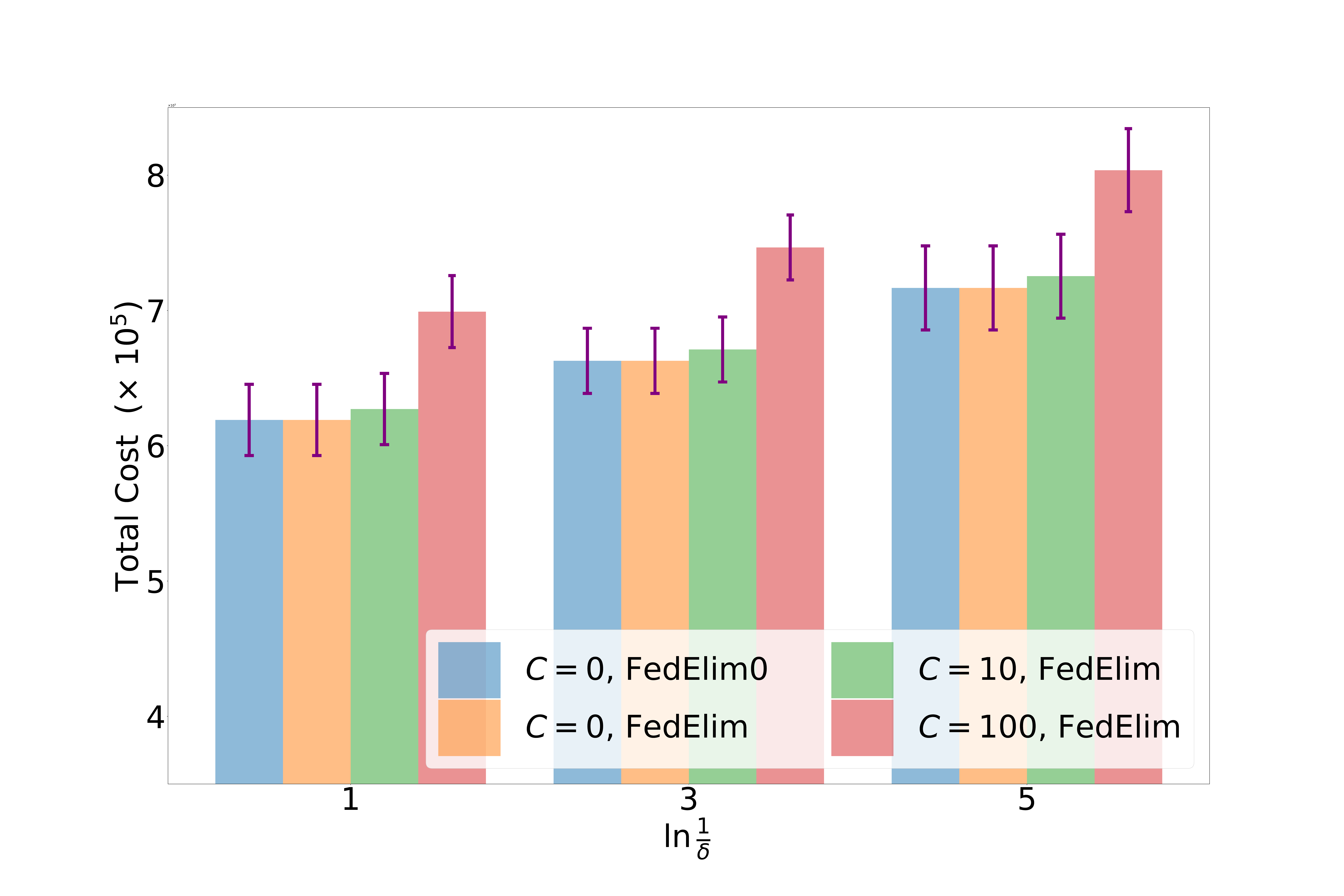}
         \caption{Plot analogous to Figure~\ref{subfig:fedelim-total-cost-and-arm-pulls}.}
         \label{subfig:fedelim-total-cost-and-arm-pulls-movielens}
     \end{subfigure}
     \hspace{.05in}
     \begin{subfigure}[b]{1.275\textwidth/4}
         \centering
         \includegraphics[width=0.87\textwidth]{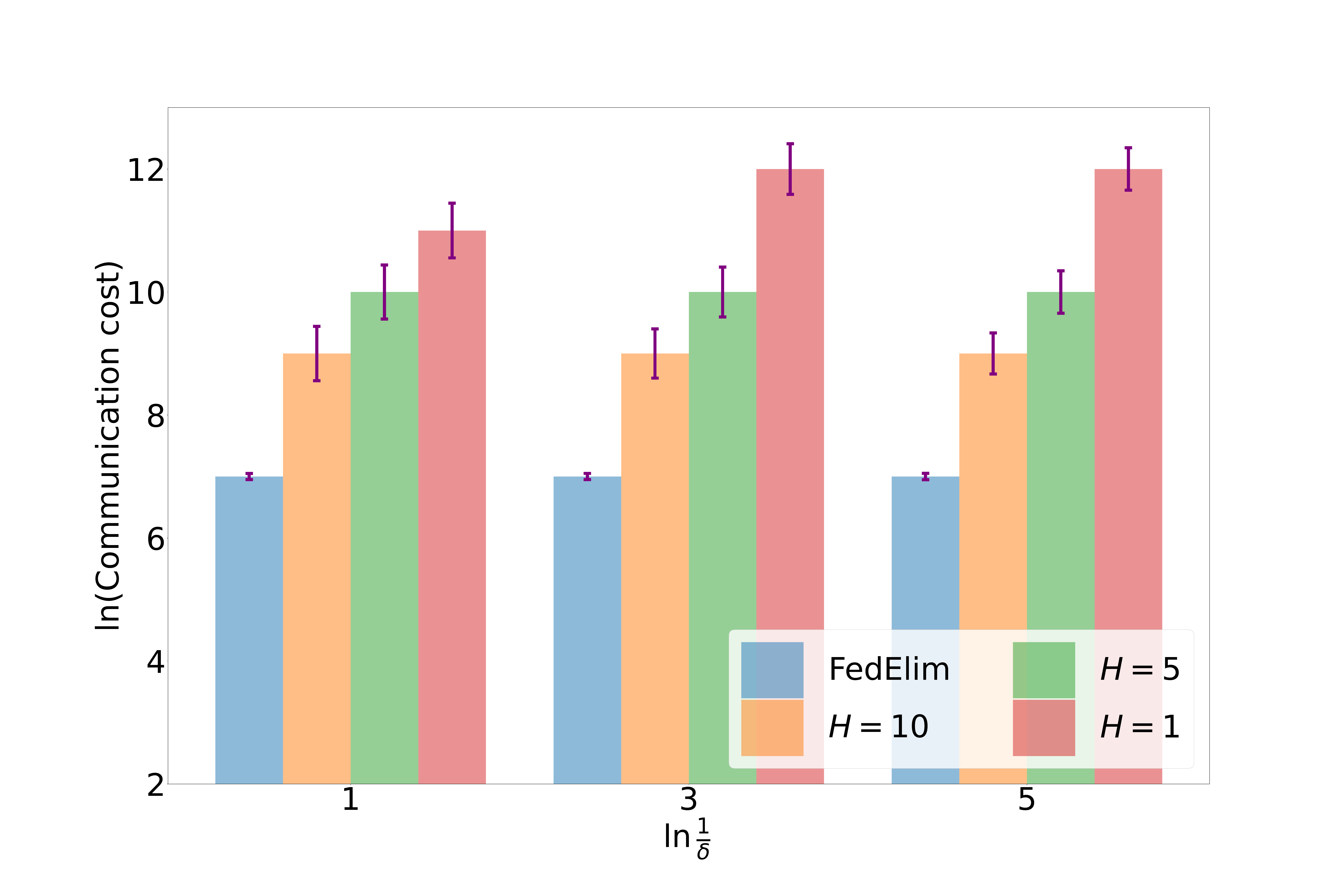}
         \caption{Plot analogous to Figure~\ref{subfig:comm-cost-comparison} with $C=10$.}
         \label{subfig:comm-cost-comparison-movielens}
     \end{subfigure}
     \hspace{.05in}
     \begin{subfigure}[b]{1.275\textwidth/4}
         \centering
         \includegraphics[width=0.87\textwidth]{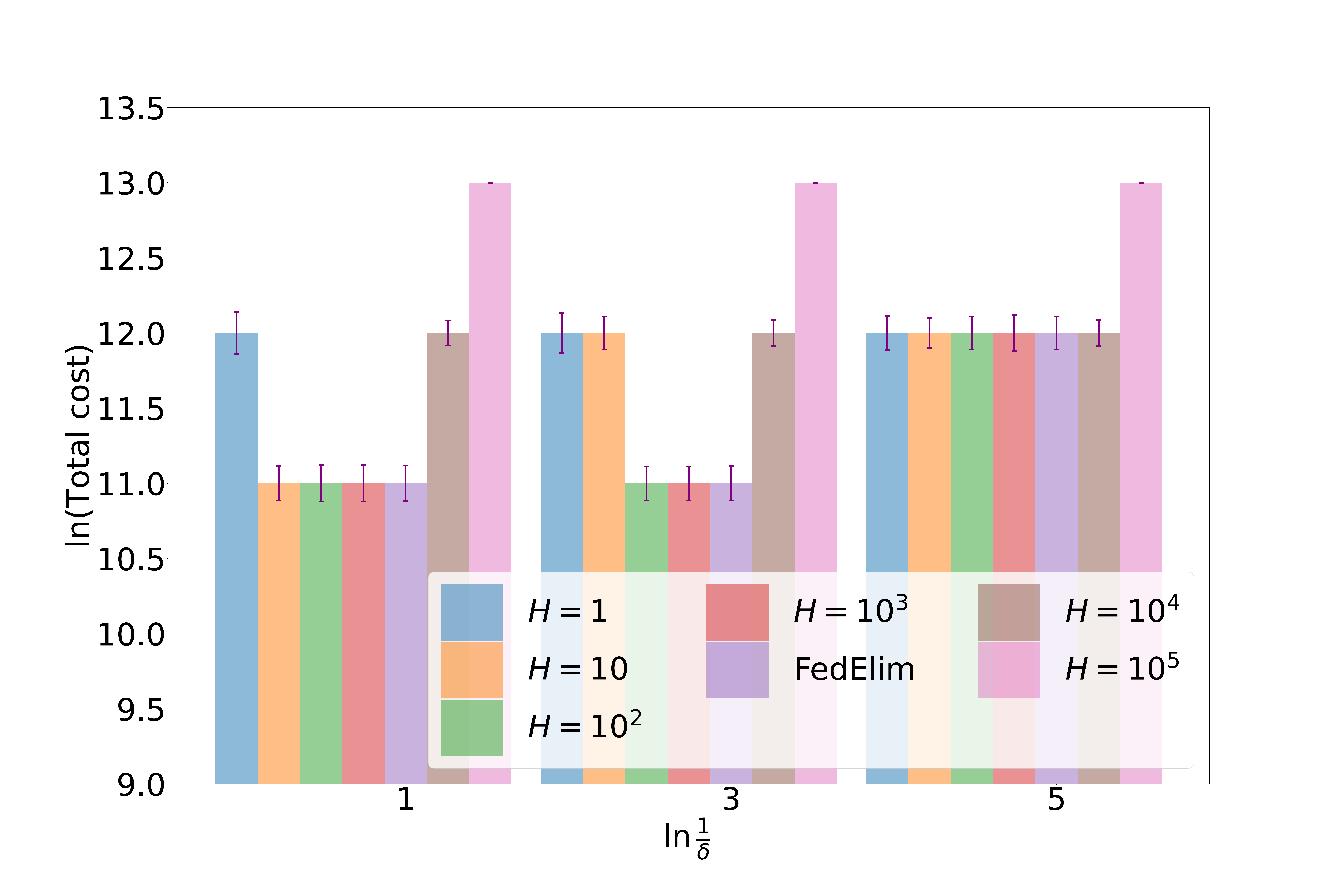}
         \caption{Plot analogous to Figure~\ref{subfig:total-cost-comparison} with $C=10$.}
         \label{subfig:total-cost-comparison-movielens}
     \end{subfigure}
        \caption{Numerical results on the MovieLens dataset.}
        \label{fig:three-figures_movielens}
\end{figure*} 
In this section, we present numerical results on the performance of {\sc FedElim} (and {\sc FedElim0}). We consider two synthetic datasets and one real-world dataset. 
\subsection{Experiments on a Synthetic Dataset}
First, we discuss our numerical results on a stylized synthetic dataset. We consider the problem instance
\begin{equation}
    \boldsymbol{\mu} = \begin{bmatrix}
	0.9 & 0.1 & 0.1\\
	0.1 & 0.9 & 0.1 \\
	0.1 & 0.1 & 0.9 \\
	0.5 & 0.5 & 0.5
\end{bmatrix} \in \mathbb{R}^{4 \times 3}. \label{eqn:syn}
\end{equation}
Notice that arm $m$ is the local best arm of client $m$ for each $m\in [3]$, whereas arm $4$ is the global best arm. 
Figure~\ref{fig:three-figures} shows a summary of the results obtained after averaging across $100$ independent trials. The error bars show $\pm 1$ standard deviation away from the mean.
Theorem \ref{thm:FLSEA_withcost} states that for  {\sc FedElim}, the total cost when $C>0$ is at most three times that when $C=0$. Figure~\ref{subfig:fedelim-total-cost-and-arm-pulls} strongly corroborates this. It shows the total cost for {\sc FedElim} for various values of $C$ as well as {\sc FedElim0}. We observe that for each fixed value of $\delta$, the total cost of {\sc FedElim}   is at most three times that of {\sc FedElim0} regardless of the value of $C$. In fact, the multiplicative factor three is conservative as empirically observed from  Figure~\ref{subfig:fedelim-total-cost-and-arm-pulls}. 

In Figure~\ref{subfig:comm-cost-comparison}, we compare the communication cost of {\sc FedElim} and periodic communication  \citep{mitra2021exploiting}. First, we see that as $H$ increases, the communication cost decreases as expected. Second, we observe that the communication cost of {\sc FedElim} is significantly smaller than that of the  periodic communication schemes. 

In Figure~\ref{subfig:total-cost-comparison}, we compare the total cost (defined in \eqref{eq:cost-of-an-algorithm}) of {\sc FedElim} and periodic communication  with period $H = 10^p$ for $p\in\{0, \ldots,  5\}$. We observe  that, per Remark~\ref{rmk:knowledge_of_C}, for periodic communication, there is a ``sweet spot'' for $H$, which, in this case, occurs at around $H=10^2$. On the other hand, {\sc FedElim} does almost as well as the best periodic communication scheme for the optimal $H$. {\sc FedElim}  is, however, agnostic to the cost $C$, which is set to $10$ here. More experimental results, specifically on a synthetic dataset of Bernoulli observations used in \citep{mitra2021exploiting}, is available in the supplemental material.

\subsection{Experiments on the MovieLens Dataset}
In addition to two synthetic datasets, we also run our algorithm on a large-scale subsampled version of the MovieLens dataset \citep{cantador2011second} crafted so as to simulate heterogeneity among the clients.
Specifically, we extract a subset of the MovieLens dataset containing movies that were produced in $38$ different countries and across $20$ different genres, resulting in a total of $8,636$ movies and about~$2.04$ million ratings (details are in the supplementary material). We then set the countries and genres to be in one-to-one correspondence with the clients (so $M=38$) and arms (so $K=20$), respectively. Figure~\ref{fig:three-figures_movielens} shows the results of running {\sc FedElim} on this dataset. The qualitative behaviours of {\sc FedElim}, {\sc FedElim0} and the strategy that communicates with the server periodically match with those for the synthetic dataset. 
The highlight of Figure~\ref{subfig:total-cost-comparison-movielens} is that {\sc FedElim} attains the {\em absolute} minimum among all schemes that communicate periodically to the server, thus  incontrovertibly demonstrating the ability of {\sc FedElim} to effectively balance communication and the number of arm selections on a real-world, large-scale dataset. 

\section{Summary and Future Work}
We have designed and analyzed an algorithm called {\sc FedElim} that is proven to be effective in learning the local best arms and the global best arm in the context of a federated learning setting. What we have not taken into account is the fact that communication to the server typically requires quantization of the empirical means at each of the clients. Elucidating the tradeoff between the number of bits used and the number of arm pulls  (cf.\ \citet{hanna2022solving}) is a promising area for future research. 

\section*{Acknowledgements}
This research/project is supported by the National Research Foundation (NRF) Singapore and DSO National Laboratories under the AI Singapore Programme (AISG Award No: AISG2-RP-2020-018) and by an NRF Fellowship (Grant No:  A-0005077-01-00).

\bibliography{aaai23-FLMAB}

\onecolumn

\begin{center}
    {\LARGE\bf Supplementary Material }
\end{center}

\subsection{Proof of Theorem~\ref{thm:correctness}}
\label{subsec:proof-of-theorem-correctness}

To prove the theorem, we use the following concentration inequality. The proof is standard and omitted.
\begin{lemma}
\label{lemma:conc-ineq}
Let $X_1,X_2,\ldots,X_n$ be $n$ {\em i.i.d.} random variables distributed according to a a Gaussian  distribution with mean $\mu$ and variance $\sigma^2$. Then, for any $\epsilon>0$, we have
\begin{equation}\mathbb{P}\Bigg(\left|\frac{\sum_{i=1}^{n}X_i}{n}-\mu \right|> \epsilon\Bigg)\leq 2 \exp \left \lbrace -\frac{ n\epsilon^2}{2\sigma^2} \right \rbrace. \label{eqn:conc}
\end{equation}
\end{lemma}
\begin{remark} \label{rmk:subG} Note that the concentration inequality in \eqref{eqn:conc} in  Lemma  \ref{lemma:conc-ineq}  also holds for sub-Gaussian random variables with variance proxy $\sigma^2$. Hence the rest of our algorithmic results carry over  {\em mutatis mutandis} to  sub-Gaussian observations.
\end{remark}

The following result asserts that under the {\sc FedElim} algorithm, with high probability, the empirical means of the arms lie within an interval around their true means, an interval whose length is specified by the local and global confidence parameters $\alpha_{\mathrm{l}}(n)$ and $\alpha_{\mathrm{g}}(n)$, where recall that $\alpha_{\mathrm{l}}(n):=\sqrt{\frac{2\ln{(8KMn^2/\delta)}}{n}}$ and $\alpha_{\mathrm{g}}(n):=\sqrt{\frac{2\ln{(8Kn^2/\delta)}}{Mn}}$ for all $n\in \mathbb{N}$.

\begin{lemma}\label{lma:confidence_bounds}
Under $\pi=\textsc{FedElim}$, let $\hat{\mu}_{k,m}(n)$ be the empirical mean of arm $k$ of client $m$ at time step $n$, and let $\hat{\mu}_{k}(n)=\sum_{m=1}^{M} \hat{\mu}_{k,m}(n)/M$. Let $\alpha_{\mathrm{l}}(n)$ and $\alpha_{\mathrm{g}}(n)$ be as defined above. Let $\mathcal{E}$ be the event in \eqref{eq:event-E}.
Then, $P^\pi_{\boldsymbol{\mu}}(\mathcal{E})\geq 1-\delta$. 
\end{lemma}

\begin{proof}
By the union bound, we have
\begin{align}
    P^\pi_{\boldsymbol{\mu}}(\mathcal{E}^c)&\leq \sum_{n=1}^{\infty}\ \sum_{m\in [M]}\ \sum_{k\in [K]}P^\pi_{\boldsymbol{\mu}}(|\hat{\mu}_{k,m}(n)-{\mu}_{k,m}(n)| > \alpha_{\mathrm{l}}(n))+\sum_{n=1}^{\infty}\ \sum_{k\in [K]}P^\pi_{\boldsymbol{\mu}}(|\hat{\mu}_{k}(n)-{\mu}_{k}(n)| > \alpha_{\mathrm{g}}(n))\\
    &\leq \sum_{n=1}^{\infty}\ \sum_{m\in [M]}\sum_{k\in [K]} \frac{\delta}{4KMn^2}+\sum_{n=1}^{\infty} \ \sum_{k\in [K]} \frac{\delta}{4Kn^2}\nonumber\\
    & \leq \delta,
    \label{eq:proof-1}
\end{align}
where the penultimate line above follows from the application of Lemma \ref{lemma:conc-ineq} to the {\em i.i.d} observations generated from each local arm of each client.
\end{proof}

With the above ingredients in place, we are now ready to prove Theorem \ref{thm:correctness}.

\begin{proof}[Proof of Theorem \ref{thm:correctness}]
From Lemma \ref{lma:confidence_bounds}, event $\mathcal{E}$ in \eqref{eq:event-E} occurs with probability at least $1-\delta$. Therefore, it suffices to prove that conditioned on the event $\mathcal{E}$, the {\sc FedElim} algorithm always give the correct output. We prove this by contradiction. Suppose that under the event $\mathcal{E}$, the algorithm outputs $\hat{k}_m^*\neq k_m^*$ for some client $m$, i.e., the algorithm's estimate of client $m$'s local best arm differs from the actual local best arm of client $m$. This implies that there exists a time step $n$ such that
\begin{align}\label{eq:elimination1}
    \hat{\mu}_{\hat{k}_m^*,m}(n)-\hat{\mu}_{k_m^*,m}(n)\geq 2\alpha_{\mathrm{l}}(n).
\end{align} 
However, on the event $\mathcal{E}$, 
\begin{align}
    \hat{\mu}_{\hat{k}_m^*,m}(n)-\hat{\mu}_{k_m^*,m}(n) &\leq \mu_{\hat{k}_m^*,m}(n)+\alpha_{\mathrm{l}}(n)-(\mu_{k_m^*,m}(n)-\alpha_{\mathrm{l}}(n))\\
    &\leq 2\alpha_{\mathrm{l}}(n)-(\mu_{k_m^*,m}(n)-\mu_{\hat{k}_m^*,m}(n))\\
    &<2\alpha_{\mathrm{l}}(n),
\end{align} 
which is contradicts~\eqref{eq:elimination1}. Therefore, it must be the case that $\hat{k}_m^* = k_m^*$ under the event $\mathcal{E}$. Similar arguments may be used to show that $\hat{k}^* = k^*$ under the event $\mathcal{E}$. This completes the proof of the theorem.
\end{proof}

\subsection{Proof of Theorem \ref{thm:FLSEA_upperbound}}
\label{subsec:proof-of-high-prob-upper-bound-c-is-0}

This section is organised as follows. First,  conditioned on the event $\mathcal{E}$ defined in \eqref{eq:event-E}, we obtain in Lemma \ref{lma:local_stoppingtime_FLSEA} an upper bound on the number of time steps of $\textsc{FedElim}0$ required to identify a local arm of client $m$ as being the local best arm of client $m$ or identify that it is not then local best arm. Then,  conditioned on~$\mathcal{E}$, we obtain in Lemma \ref{lma:global_stoppingtime_FLSEA} an upper bound on the number of time steps of $\textsc{FedElim}0$ required to identify an arm as being the global best arm or otherwise. We complete the proof using  Lemmas~\ref{lma:local_stoppingtime_FLSEA} and~\ref{lma:global_stoppingtime_FLSEA}.

Note that even though Theorem \ref{thm:correctness} says that {\sc FedElim} algorithm is $\delta$-PAC, by following the same proof steps, we can prove that the $\textsc{FedElim}0$ algorithm is also $\delta$-PAC and that the event $\mathcal{E}$ holds with probability at least $1-\delta$ under $\textsc{FedElim}0$. A careful reader may observe that the algorithms $\textsc{FedElim}0$ and {\sc FedElim} differ only in their   communication rounds, but neither in the stopping rule nor in the arm selection/elimination rule.

\begin{lemma}\label{lma:local_stoppingtime_FLSEA}
 Let $T^{(1)}_{k,m}$ denote the time step of $\textsc{FedElim}0$ in which arm $k$ of client $m$ is identified as the local best arm of client $m$ or otherwise. Let $\mathcal{E}$ be the event defined in \eqref{eq:event-E}. Conditioned on the event $\mathcal{E}$, $T_{k,m}^{(1)} \leq T_{k,m}$, where $T_{k,m}$ is as defined in \eqref{eq:T_km}.
 \end{lemma}
 
 \begin{proof}
Consider a local arm  $k$ of client $m$ that is not the best arm (non-best arm) of client $m$, i.e., $k\neq k_m^*$. Then, because the $\textsc{FedElim}0$ algorithm identifies $k$ as a local non-best arm of client $m$ in time step $T^{(1)}_{k,m}$, it must be the case that on the event~$\mathcal{E}$, $$\hat{\mu}_{*,m}\big(T^{(1)}_{k,m}\big)-\hat{\mu}_{k,m}\big(T^{(1)}_{k,m}\big)\geq 2\alpha_{\mathrm{l}}\big(T^{(1)}_{k,m}\big).$$ Notice that $\alpha_{\mathrm{l}}(n) = \sqrt{\frac{2\ln{(8KMn^2/\delta)}}{n}} \to 0$ as $n\to\infty$. Let $n^*_{k,m} \coloneqq \inf\big\lbrace n:  \alpha_{\mathrm{l}}(n^\prime) \leq \frac{\Delta_{k,m}}{4 }\ \forall\, n^\prime\geq n \big\rbrace$. Then, it must be the case that $\hat{\mu}_{k_m^*,m}(n)-\hat{\mu}_{k,m}(n)\geq 2\alpha_{\mathrm{l}}(n)$ for all $n\geq n^*_{k,m}$, from which it follows that $T^{(1)}_{k,m}\leq n^*_{k,m}$  on the event~$\mathcal{E}$. 

We now obtain an upper bound on $n^*_{k,m}$.
Let $$n_{k,m} \coloneqq \left\lceil \max \left\lbrace x\in (1, \infty): \alpha_{\mathrm{l}}(x) = \frac{\Delta_{k,m}}{4} \right\rbrace \right\rceil. $$ The maximum in the above equation picks the larger of the two solutions to the equation $\alpha_{\mathrm{l}}(x)=\Delta_{k, m}/4$ in the range $x\in (1, \infty)$, and the ceil function   $\lceil x \rceil$ returns the smallest integer greater than or equal to $x$. Clearly, $n^*_{k,m} \leq n_{k,m}$. Letting $$a=\frac{\Delta_{k,m}^2}{64}\quad\mbox{and}\quad b=\frac{\ln\left(\frac{8KM}{\delta}\right)}{2},$$ the exact expression for $n_{k,m}$ is given by $n_{k,m}=\left\lceil- \frac{ 1}{a}W_{-1}(-ae^{-b})\right \rceil$, where 
 $W_{-1}(y)$, $y<0$, is the smallest value of $x$ such that $x e^x=y$ holds ($W_{-1}$ is known as the {\em Lambert $W $ function}). 
From \citet[Theorem 3.1]{alzahrani2018sharp}, we have $W_{-1}(y)>\frac{e}{e-1}\ln(-y)$, and therefore
\begin{align}
	T^{(1)}_{k,m} &\leq n_{k,m} \nonumber\\
	&\leq \left \lceil \frac{e}{e-1}\cdot \frac{b-\ln(a)}{a} \right \rceil \nonumber\\
	&\leq \frac{e}{e-1}\cdot \frac{b-\ln(a)}{a} + 1 \nonumber\\
	&=\frac{64 e}{e-1} \cdot \frac{1}{\Delta_{k,m}^2} \cdot \ln\left(\frac{64\sqrt{\frac{8KM}{\delta}}}{\Delta_{k,m}^2}\right) + 1\nonumber\\
	&=102 \cdot \frac{1}{\Delta_{k,m}^2} \cdot \ln\left(\frac{64\sqrt{\frac{8KM}{\delta}}}{\Delta_{k,m}^2}\right) + 1.
\end{align}
Finally, we note that the local best arm of client $m$ is identified when each of the other arms of client $m$ is identified as a non-best arm. Therefore, is identified as the best arm when all the other sub-optimal arms are inactive. Therefore, 
\begin{align}
	T^{(1)}_{k^*_m,m} &= \max_{k\neq k_m^*} T^{(1)}_{k,m} \leq
	102 \cdot \frac{1}{\Delta_{k_m^*,m}^2} \cdot \ln\left(\frac{64\sqrt{\frac{8KM}{\delta}}}{\Delta_{k_m^*,m}^2}\right) + 1.
\end{align}
This completes the proof of Lemma \ref{lma:local_stoppingtime_FLSEA}.
 \end{proof}
 
 \begin{lemma}\label{lma:global_stoppingtime_FLSEA}
Let $T^{(2)}_{k}$ denote the time step of $\textsc{FedElim}0$ in which arm $k$ is identified as the global best arm at the server or otherwise. Let $\mathcal{E}$ be the event defined in \eqref{eq:event-E}. Conditioned on $\mathcal{E}$, $T_{k}^{(2)} \leq T_{k}$, where $T_{k}$ is as defined in \eqref{eq:T_k}
 \end{lemma}
 
 \begin{proof}
The proof follows along the lines of the proof of Lemma \ref{lma:local_stoppingtime_FLSEA} and is thus omitted.
 \end{proof}
With the above ingredients in place, we now prove Theorem \ref{thm:FLSEA_upperbound}.
 
\begin{proof}[Proof of Theorem \ref{thm:FLSEA_upperbound}]
Let $T^*_{k,m}$ denote the total number of selections of arm $k$ of client $m$ under $\textsc{FedElim}0$ up to stoppage. From Lemma \ref{lma:local_stoppingtime_FLSEA}, we know that under the event $\mathcal{E}$, arm $k$ can be identified as the local best arm of client $m$ or otherwise after $T^{1}_{k,m}$ selections of arm $k$. Also, from Lemma \ref{lma:global_stoppingtime_FLSEA}, we know that under the event $\mathcal{E}$, arm $k$ can be identified as the global best arm or otherwise after $T^{(2)}_{k}$ selections of arm $k$. Thus, it follows that $T^*_{k,m}=\max\{T^{(1)}_{k,m},T^{(2)}_{k}\}$ on the event $\mathcal{E}$ for all $k\in [K]$ and $m \in [M]$.
 \begin{align}
     T_{\textsc{FedElim}0} &= \sum_{k=1}^{K}\ \sum_{m=1}^{M} T^*_{k,m}\nonumber\\
     &= \sum_{k=1}^{K}\ \sum_{m=1}^{M} \max\{T^{(1)}_{k,m},T^{(2)}_{k}\} \nonumber\\
     &\leq \sum_{k=1}^{K}\ \sum_{m=1}^{M} \max\{T_{k,m},T_{k}\},
 \label{eq:proof-4}
 \end{align}
where the inequality above follows from Lemmas \ref{lma:local_stoppingtime_FLSEA} and \ref{lma:global_stoppingtime_FLSEA}. The desired result is thus proved.
 \end{proof}
 
 \subsection{Proof of Theorem \ref{thm:lb1_c0}}
 \label{subsec:proof-of-thm-lower-bound-1}
 \begin{proof}
 Let us arbitrarily fix $\delta>0$, a problem instance $\boldsymbol{\mu}=[\mu_{k,m}:k\in [K],m\in[M]]$, and a $\delta$-PAC algorithm $\pi$. Let  $N_{k,m}$ denote the (random) number of selections of arm $k$ of client $m$ under $\pi$ up to its termination time. Let us define  $\mathrm{Alt}(\boldsymbol{\mu})$, the set of all problem instances {\em alternative} to $\boldsymbol{\mu}$, as
\begin{equation}
	\mathrm{Alt}(\boldsymbol{\mu}) \coloneqq \{\boldsymbol{\mu}'\in \mathbb{R}^{K\times M}:\mathbf{S}(\boldsymbol{\mu})\neq \mathbf{S}(\boldsymbol{\mu}') \}.
	\label{eq:alternative-problem-instances}
\end{equation}
From \citet[Lemma 1]{kaufmann2016complexity} and using the fact that the Kullback--Leibler divergence between two unit-variance Gaussian distributions with means $\mu$ and $\mu'$ is $d(\mu,\mu')=(\mu-\mu')^2/2$, we have
\begin{align}
	\sum_{k=1}^{K}\ \sum_{m=1}^{M}\ \mathbb{E}_{\boldsymbol{\mu}}^\pi [N_{k,m}]\ (\mu_{k,m}-\mu_{k,m}')^2
	\geq 2 \ln \left( \frac{1}{2.4\delta} \right)
	\label{eq:transportation-lemma}
\end{align}
for all $\boldsymbol{\mu}^\prime \in \mathrm{Alt}(\boldsymbol{\mu})$. 

Fix $\varepsilon>0$ and $m\in [M]$ arbitrarily, and suppose that $k_m^*\in [K]$ is the local best arm of client $m$ under the problem instance $\boldsymbol{\mu}$. Fix an arbitrary $k \neq k_m^*$. Let  $\boldsymbol{\mu}^{(1)}=[\mu_{k^\prime,m^\prime}^{(1)}:k^\prime\in [K],m^\prime\in[M]]$ be a problem instance such that $\mu_{k^\prime,m^\prime}^{(1)}=\mu_{k^\prime,m^\prime}$ for all $(k^\prime,m^\prime)\neq (k,m)$ and $\mu_{k,m}^{(1)}=\mu_{k^*_{m},m}+\varepsilon$. Clearly, then, arm $k$ is the local best arm of client $m$ under the problem instance $\boldsymbol{\mu}^{(1)}$, and therefore $\boldsymbol{\mu}^{(1)}\in\mathrm{Alt}(\boldsymbol{\mu})$. Applying the inequality in \eqref{eq:transportation-lemma} to $\boldsymbol{\mu}^{(1)}$, we get
\begin{equation}
	\mathbb{E}_{\boldsymbol{\mu}}^\pi [N_{k,m}]\ (\mu_{k,m}-\mu_{k^*_{m},m}-\varepsilon)^2 \geq 2\ln \left(\frac{1}{2.4\delta}\right),
	\label{eq:proof-5}
\end{equation}
or equivalently,
\begin{equation}
	\mathbb{E}_{\boldsymbol{\mu}}^\pi [N_{k,m}]\geq \frac{2\ln\left(\frac{1}{2.4\delta}\right)}{(\Delta_{k,m}+\varepsilon)^2}.
	\label{eq:proof-6}
\end{equation}
Letting $\varepsilon\downarrow 0$, we get 
\begin{align}
    \mathbb{E}_{\boldsymbol{\mu}}^\pi [N_{k,m}]\geq \frac{2\ln\left(\frac{1}{2.4\delta}\right)}{\Delta_{k,m}^2},\quad k \neq k_m^*.
    \label{eq:localsubopt_lb}
\end{align}

To derive a bound similar to \eqref{eq:localsubopt_lb} for $k_m^*$, fix an arbitrary $\varepsilon>0$ and let $\boldsymbol{\mu}^{(2)}=[\mu_{k^\prime,m^\prime}^{(2)}:k^\prime\in [K],m^\prime\in[M]]$ be a problem instance such that $\mu_{k^\prime,m^\prime}^{(2)}=\mu_{k^\prime,m^\prime}$ for all  $(k^\prime,m^\prime)\neq (k^*_{m},m)$ and $\mu_{k^*_{m},m}^{(2)}=\max_{k^\prime \neq k^*_{m}}\mu_{k^\prime,m}-\varepsilon$. Clearly, then, the local best arm of client $m$ under the instance $\boldsymbol{\mu}^{(2)}$ is not arm $k^*_{m}$. Therefore,  $\boldsymbol{\mu}^{(2)}\in\mathrm{Alt}(\boldsymbol{\mu})$, and
\begin{equation}
	\mathbb{E}_{\boldsymbol{\mu}}^\pi [N_{k,m}]\ (\mu_{k_m^*,m}- \mu_{k^*_{m},m}^{(2)})^2 \geq \ln \left(\frac{1}{2.4\delta}\right),
	\label{eq:proof-7}
\end{equation}
or equivalently,
\begin{equation}
	\mathbb{E}_{\boldsymbol{\mu}}^\pi [N_{k_m^*,m}] \geq \frac{2\ln\left(\frac{1}{2.4\delta}\right)}{(\Delta_{k^*_{m},m}+\varepsilon)^2}.
	\label{eq:proof-8}
\end{equation}
As before, letting $\varepsilon \downarrow 0$, we get 
\begin{align}
    \mathbb{E}_{\boldsymbol{\mu}}^\pi [N_{k_m^*,m}]\geq\frac{2\ln\left(\frac{1}{2.4\delta}\right)}{\Delta_{k^*_{m},m}^2}.
    \label{eq:localopt_lb}
\end{align}

Fix $\varepsilon>0$ and $m\in [M]$ arbitrarily. Let $k^*$ be the global best arm under the problem instance $\boldsymbol{\mu}$. Fix $k \neq k^*$ arbitrarily. Let $\boldsymbol{\mu}^{(3)}=[\mu_{k^\prime,m^\prime}^{(3)}:k^\prime\in [K],m^\prime\in[M]]$ be a problem instance such that $\mu_{k^\prime,m^\prime}^{(3)}=\mu_{k^\prime,m^\prime}$ for all $(k^\prime,m^\prime)\neq (k,m)$ and $\mu_{k,m}^{(3)}=\sum_{m^\prime\in[M]}\mu_{k^*,m^\prime}-\sum_{m^\prime\in[M]\backslash\{m\}}\mu_{k,m^\prime}+\varepsilon$. Clearly, $k$ is the global best arm under the instance $\boldsymbol{\mu}^{(3)}$. Therefore,  $\boldsymbol{\mu}^{(3)}\in\mathrm{Alt}(\boldsymbol{\mu})$, and it follows from \eqref{eq:transportation-lemma} that
\begin{equation}
	\mathbb{E}_{\boldsymbol{\mu}}^\pi [N_{k,m}] \geq \frac{2\ln\left(\frac{1}{2.4\delta}\right)}{M^2(\Delta_{k}+\varepsilon/M)^2}.
	\label{eq:proof-10}
\end{equation}
Letting $\varepsilon\downarrow 0$, we get 
\begin{align}
    \mathbb{E}_{\boldsymbol{\mu}}^\pi [N_{k,m}] \geq \frac{2\ln\left(\frac{1}{2.4\delta}\right)}{M^2\Delta_{k}^2}, \quad k \neq k^*.
    \label{eq:globalsubopt_lb}
\end{align}

Finally, consider the global best arm $k^*$ and client $m$. By following a similar procedure as above to generate a problem instance $\boldsymbol{\mu}^{(4)}$ such that the mean of the global best arm under the instance $\boldsymbol{\mu}^{(4)}$ is equal to $\max_{k^\prime\neq k^*}\mu_{k^\prime}$ (i.e., the mean of the global second best arm under the instance $\boldsymbol{\mu}$), we get 
\begin{align}
    \mathbb{E}_{\boldsymbol{\mu}}^\pi [N_{k^*,m}] \geq \frac{2\ln\left(\frac{1}{2.4\delta}\right)}{M^2\Delta_{k^*}^2}.
    \label{eq:globalopt_lb}
\end{align}
Combining the lower bounds in~\eqref{eq:localsubopt_lb},  \eqref{eq:localopt_lb}, \eqref{eq:globalsubopt_lb},  and~\eqref{eq:globalopt_lb}, we obtain that for all $\delta$-PAC algorithms $\pi$,
\begin{equation}
	 \mathbb{E}_{\boldsymbol{\mu}}^\pi [T_{\delta}^{\pi}] \geq \sum_{k=1}^{K}\ \sum_{m=1}^{M}\mathbb{E}_{\boldsymbol{\mu}}^\pi [N_{k,m}] \geq \sum_{k=1}^{K}\ \sum_{m=1}^{M}\max \left \lbrace \frac{2\ln\left(\frac{1}{2.4\delta}\right)}{\Delta_{k,m}^2},\frac{2\ln\left(\frac{1}{2.4\delta}\right)}{M^2\Delta_{k}^2} \right \rbrace.
	 \label{eq:proof-11}
\end{equation}
This proves the desired result.
 \end{proof}
 
 \begin{remark} \label{rmk:sG_lb}
 {\color{black} The adept reader may ask whether the above proof technique carries over to sub-Gaussian distributions. If the distributions of the local arms of the clients belong to the class of all distributions that are supported on an unbounded set and are sub-Gaussian, the above proof applies since by \citet[Lemma~6.3]{Csi06}, the KL divergence between two distributions having the same variance can be upper bounded by the square of the difference of their means. However, for {\em bounded} sub-Gaussian distributions, say supported on $[0,1]$, the constructions of the alternative instances $\boldsymbol{\mu}^{(3)}$ and $\boldsymbol{\mu}^{(4)}$ in the above proof result in arm means that potentially do not lie in $[0,1]$. 
A careful construction of the alternative instances is necessary for {\em bounded} sub-Gaussian distributions.}
 \end{remark}
 
\subsection{Proof of Theorem \ref{thm:FLSEA_withcost}}
\label{subsec:proof-of-theorem-FLSEA-withcost}
\begin{proof}
Fix $k\in [K]$ and $m\in [M]$ arbitrarily. Let $T_{k,m}^{(1),C}$ denote the time step of {\sc FedElim} in which arm $k$ of client $m$ is identified as the local best arm of client $m$ or identified to be not the best arm. Because the calculation of the empirical means of the arms is independent of the value of $C$, it follows that $T_{k,m}^{(1),C}=T_{k,m}^{(1)}$, where $T_{k,m}^{(1)}$ is as defined in Lemma \ref{lma:local_stoppingtime_FLSEA}. Invoking the upper bound for $T_{k,m}^{(1)}$ from Lemma \ref{lma:local_stoppingtime_FLSEA}, we have
\begin{align}
 	T_{k,m}^{(1),C}=T^{(1)}_{k,m} &\leq 102 \cdot \frac{1}{\Delta_{k,m}^2}\cdot \ln\left(\frac{64\sqrt{\frac{8KM}{\delta}}}{\Delta_{k,m}^2} \right) + 1
	\label{eq:proof-2C}
 \end{align}
under the event $\mathcal{E}$. Let $T_{k}^{(2),C}$ denote the time step of $\textsc{FedElim}$ in which arm $k$  is identified at the server as the global best arm or otherwise. Because the communication between the clients and the sever takes place only in time steps $n$ of the form $n\in \{2^t:t\in \mathbb{N}_0\}$, it must be the case that $T_{k}^{(2),C}=2^{t^*}$ for some $t^*\in \mathbb{N}_0$. Furthermore, from Lemma \ref{lma:global_stoppingtime_FLSEA}, it must be the case that $2^{t^*-1}< T_{k}^{(2)}$  or 
$T_{k}^{(2),C}< 2\  T_{k}^{(2)}$.
 Thus,
 \begin{align}\label{eq:armpulls_FLSEA}
    T_{\textsc{FedElim}}^C= & \sum_{k=1}^{K}\ \sum_{m=1}^{M} \max\{T_{k,m}^{(1),C}, T_k^{(2),C}\}\nonumber\\
    &\leq  \sum_{k=1}^{K}\ \sum_{m=1}^{M} \max\{T_{k,m}^{(1)}, 2\, T_k^{(2)}\}\nonumber\\
    &\leq \sum_{k=1}^{K}\ \sum_{m=1}^{M} \max\{T_{k,m}, 2\, T_k\}\nonumber\\
    &\leq \sum_{k=1}^{K}\ \sum_{m=1}^{M} \max\{2\,T_{k,m}, 2\, T_k\}\nonumber\\
    &= 2\,T ,
 \end{align}
 where $T$ is defined in \eqref{eq:high-prob-upper-bound-c-is-zero}.
 
Notice that client $m$ communicates the empirical mean updates of its local arm $k$ to the server in the time steps $n=1,2,\ldots,2^{t^*}=T_{k}^{(2),C}$. Hence, the total communication cost incurred by client $m$ to communicate the empirical mean of its local arm $k$ is $C \, t^*= C \, \log_2{T_{k}^{(2),C}}\leq C \, \left\lceil\log_2{T_{k}}\right\rceil$. Therefore, on the event $\mathcal{E}$, the total communication cost aggregated across all the local arms of all the clients is 
 \begin{equation}\label{eq:communication_cost}
     C_{\textsc{FedElim}}^{\mathrm{comm}} \leq C \sum_{k=1}^{K}\ \sum_{m=1}^{M}\left\lceil \log_2{T_{k}}\right\rceil.
 \end{equation}
From \eqref{eq:armpulls_FLSEA} and \eqref{eq:communication_cost}, it follows that under the event $\mathcal{E}$,
 \begin{align}\label{eq:total_cost}
 	C_{\textsc{FedElim}}^{\mathrm{total}} &= T_{\textsc{FedElim}}^C+C_{\textsc{FedElim}}^{\mathrm{comm}} \nonumber\\
	&\leq \sum_{k=1}^{K}\sum_{m=1}^{M} \big( \max\{T_{k,m}, 2\, T_k\}+C\, \left\lceil \log_2 {T_{k}}\right\rceil\big) \nonumber\\
	&\leq 3\,T.
 \end{align}
This proves the desired result.
\end{proof}


\section{Exponentially Sparse Communication with Base $\lambda$}
Recall our exponentially sparse communication scheme wherein communication between the clients and the server takes place only in time steps $n=2^t$ for $t\in \mathbb{N}_0$. As alluded to in the paragraphs trailing the statement of Theorem \ref{thm:FLSEA_withcost}, it is no coincidence that the constant $2$ appears inside the maximum in \eqref{eqn:first_line} and also in the denominator in \eqref{eq:comm-cost-under-FedElim}. The reader may question the rationale behind choosing $2$ as the base in the exponentially sparse communication time steps $n=2^t$ for $t\in \mathbb{N}_0$. As such, $2$ is the smallest integer base that results in sparser communication than communication in every time step. Going a step further, the reader might ask if there are any obvious advantages of communicating in time steps $n=\lceil \lambda^t \rceil$ for some $\lambda \neq 2$. 
 It suffices to investigate this question for $\lambda \geq 1$, as exponentially sparse communication with $\lambda < 1$ is clearly inferior to communicating in every time step.

Following our proof techniques, it can be shown that the $\lambda$-analogues of \eqref{eqn:first_line}-\eqref{eq:total-cost-FedElim} for any $\lambda\geq 1$  are as follows.
%
%
%
%
\begin{align}
    T_{\textsc{FedElim}}^C &\leq \sum_{k=1}^{K}\ \sum_{m=1}^{M} \max\{T_{k,m},\ \lambda \,T_k\} \leq \lambda \,T, \label{eq:lambda-analogue-1}\\
    	C_{\textsc{FedElim}}^{\mathrm{comm}} & \leq C\cdot M \cdot \sum_{k=1}^{K} \left\lceil \frac{\ln T_k }{\ln \lambda } \right\rceil \leq C\cdot M \cdot \sum_{k=1}^{K} \frac{\ln T_k }{\ln \lambda } +C\cdot M \cdot K, \label{eq:lambda-analogue-2}\\
    	C_{\textsc{FedElim}}^{\mathrm{total}} &= T_{\textsc{FedElim}}^C + C_{\textsc{FedElim}}^{\mathrm{comm}} \leq \lambda \,T + C\cdot M \cdot \sum_{k=1}^{K} \frac{\ln T_k }{\ln \lambda } +C\cdot M \cdot K. \label{eq:lambda-analogue-3}
\end{align}
The second inequality in \eqref{eq:lambda-analogue-2} results from  the  bound $\lceil x \rceil < x+1$. Notice that the first term on the right-hand side of the inequality in \eqref{eq:lambda-analogue-3} is monotonically increasing in $\lambda$, whereas the second term on the right-hand side of the inequality is monotonically decreasing in $\lambda$, thereby hinting that there may be a ``sweet spot'' for $\lambda$ where the upper bound in \eqref{eq:lambda-analogue-3} is minimal. Indeed, optimising the upper bound in \eqref{eq:lambda-analogue-3} with respect to $\lambda$ to obtain the tightest possible upper bound for $C_{\textsc{FedElim}}^{\mathrm{total}}$, we obtain that the optimiser, say  $\lambda^*$, is the unique solution to $\lambda^*(\ln \lambda^*)^2=\frac{C\cdot M}{T} \cdot \sum_{k=1}^{K}\ln T_k$. While the closed-form expression for $\lambda^*$ may not be readily available, it may nevertheless be evaluated numerically (via, for e.g., line search). Notice that $\lambda^*$ is, in general, a function of $C$, $T$ and $T_k$ for all $k\in [K]$. From the expressions for $T$ and $T_k$ in \eqref{eq:high-prob-upper-bound-c-is-zero} and \eqref{eq:T_k} respectively, it is clear that these constants are functions of the underlying problem instance $\boldsymbol{\mu}$. It is therefore overwhelmingly clear from the preceding analysis that $\lambda^*$ is, in general, a function of $C$ and other problem instance-specific constants. The latter are not typically known beforehand in most practical settings, thereby making the computation of $\lambda^*$ infeasible. Therefore, we do away with the computation of the optimal $\lambda$ and simply use $\lambda=2$, which, as we have seen, works well in theory and in practice.

%
%

%

\section{On the Optimal Period of a Periodic Communication Scheme}
Figures~\ref{subfig:total-cost-comparison} and~\ref{subfig:total-cost-comparison-movielens} indicate that the total cost of a periodic communication algorithm (based on successive elimination) with period $H$ decreases, attains a minimum, and thereafter increases as  $H$ increases, thereby suggesting that there is ``sweet spot'' for $H$, say $H_{\textsf{opt}}$, in which the total cost is minimal. In particular, Figure~\ref{subfig:total-cost-comparison} suggests that periodic communication schemes with periods $H=100$ and $H=1000$ perform better than our exponentially sparse communication for $C=10$ and the problem instance $\boldsymbol{\mu}$ in \eqref{eqn:syn}. A natural question is: {\em Can we determine the optimal value of $H$ where the total cost is minimal? } Our answers below are only partial and a possible first attempt at understanding the rather interesting trend for the total cost of periodic communication scheme observed in Figure~\ref{subfig:total-cost-comparison} (when each arm generates Gaussian rewards). 

Following our proof techniques, it can be shown that when the uplink communication cost is $C$, the analogues of \eqref{eqn:first_line}--\eqref{eq:total-cost-FedElim} for a periodic communication scheme with period $H$ are as follows:
\begin{align}
    T_{\textsc{Periodic}}^{H,C} &\leq \sum_{k=1}^{K}\ \sum_{m=1}^{M} \max\left\{T_{k,m}, \left\lceil \frac{T_k}{H} \right\rceil H\right\} \leq T + H\cdot M \cdot K, \label{eq:periodic-1} \\
    	C_{\textsc{Periodic}}^{\mathrm{comm}} & \leq C\cdot M \cdot \sum_{k=1}^{K} \left\lceil \frac{T_k}{H} \right\rceil \leq C \cdot \frac{ T }{H } +C\cdot M \cdot K, \label{eq:periodic-2}\\
    	C_{\textsc{Periodic}}^{\mathrm{total}} &=T_{\textsc{Periodic}}^{H,C} + C_{\textsc{Periodic}}^{\mathrm{comm}} \leq T + H\cdot M \cdot K + C \cdot \frac{ T }{H } +C\cdot M \cdot K. \label{eq:periodic-3}
\end{align}
The inequalities in \eqref{eq:periodic-1} and \eqref{eq:periodic-2} result from using $\lceil x \rceil \leq x+1$. Notice that the second term on the right-hand side of the inequality in \eqref{eq:periodic-3} is monotonically increasing in $H$ whereas the third term is monotonically decreasing in $H$, thereby hinting that there may be a sweet spot for $H$ where the upper bound in~\eqref{eq:periodic-3} is minimal. Indeed, differentiating the upper bound in~\eqref{eq:periodic-3} with respect to $H$, and setting this derivative to zero, we get that the optimal $H$ is $H_{\textsf{opt}}=\sqrt{\frac{C\cdot T }{M \cdot K}}$, which is clearly a function of $C$ and $T$ (a problem instance-specific constant). Whereas operating at $H_{\textsf{opt}}$ surely leads to the tightest possible upper bound on $C_{\textsc{Periodic}}^{\mathrm{total}}$ among the class of all periodic communication schemes, as is clear from the above exposition, the determination of $H_{\textsf{opt}}$ requires the knowledge of $C$ and $T$, both of which are not available in most practical settings. 

One key takeaway from Figure~\ref{subfig:total-cost-comparison} is that our exponentially sparse communication scheme achieves a similar total cost as that of a periodic communication scheme operating close to $H_{\textsf{opt}}$, albeit without the knowledge of $C$ and $T$. In fact, from Figure~\ref{subfig:total-cost-comparison-movielens}, we observe that in the experiment on the large-scale, real-world MovieLens dataset, {\sc FedElim} attains the smallest total cost among all the candidate periods $H$ of the periodic communication scheme.

\section{A Super-Exponentially Sparse Communication Scheme: Better or Worse?}

In this section, we analyse a {\em super-exponentially sparse} communication scheme in which communication takes place in time steps $n=2^{2^t}$ for $t \in \mathbb{N}_0$. Specifically, we seek  answers to the following question: {\em How does the total cost of a super-exponentially sparse communication scheme compare with those of the periodic and the exponentially sparse communication schemes?}
Intuitively, the more frequently communication takes place, the larger is the communication cost, the faster is the identification of the global best arm because of the frequent exchange of information between the clients and the server, and therefore the smaller is the total number of arm selections required. 

Table~\ref{tab:comparative-study} presents  worst-case upper bounds on the total number of arm selections, communication cost, and the total cost between (a) a periodic communication scheme with period $H$ and based on successive elimination, (b) {\sc FedElim} (with exponentially sparse communication), and (c) a super-exponentially sparse communication scheme based on successive elimination. 
From the upper bounds in Table \ref{tab:comparative-study}, we see that at one extreme of the total cost spectrum lies the periodic communication scheme which achieves a small total number of arm selections but incurs a large communication cost ($O(T)$). At the other extreme lies the super-exponentially sparse communication scheme in which the communication cost is small ($O(\ln \ln T)$) but the total number of arm selections is large ($O(T^2)$).  {\sc FedElim}, which utilises an   exponentially sparse communication scheme, sits between the other two schemes and strikes a balanced trade-off between the total number of arm selections ($O(T)$) and communication cost ($O(\ln T)$).

\begin{table}[!t]
\centering
\begin{tabular}{|l|l|l|l|}
\hline
\small Scheme& \small No. of Arm Selections & \small Comm. Cost & \small Total Cost\\
\hline
Periodic (see \eqref{eq:periodic-1}-\eqref{eq:periodic-3}) & $ O(T)$ & $O(C\,T/H) $ &$O((1+C/H)T) $ \\
\hline
{\sc FedElim} (see \eqref{eqn:first_line}-\eqref{eq:total-cost-FedElim}) & $2\cdot T$& $O(C\,\ln T)$ & $3\cdot  T$\\
\hline
Super-Exponential & $ O(T^2)$& $O(C \, \ln \ln T)$& $O(T^2)$\\
\hline
\end{tabular}
\caption{Worst-case upper bounds on the total number of arm selections, communication cost, and the total cost of (a) a periodic communication scheme with period $H$ based on successive elimination, (b) $\textsc{FedElim}$ (with exponentially sparse communication), and (c) a super-exponentially sparse communication scheme based on successive elimination. Here, $T$ is as defined in \eqref{eq:high-prob-upper-bound-c-is-zero}.}
\label{tab:comparative-study}
\end{table}

\section{The MovieLens Dataset: Description, Dataframes, Cleanup, and Sampling}
For our numerical experiments, we use the publicly available MovieLens dataset available at \url{https://files.grouplens.org/datasets/hetrec2011/hetrec2011-movielens-2k-v2.zip}. We extract the movie ratings from the $\texttt{user-ratedmovies.dat}$ file, the country names from the $\texttt{movie-countries.dat}$ file, and the genres from the $\texttt{movie-genres.dat}$ file. Common to each of the aforementioned files is the \texttt{movieID} semantic, a unique identifier associated with each movie for which one or more users' ratings are available in the dataset. There are a total of $10,197$ distinct \texttt{movieID} values in the dataset. 

\begin{table}[t]
\centering
\begin{tabular}{ |l|l| } 
 \hline
 Afghanistan & Libya \\ 
 Bhutan & Mexico \\ 
 Bosnia and Herzegovina & Norway \\ 
 Burkina Faso & Peru \\ 
 Canada & Philippines \\ 
Chile & Poland \\ 
Croatia & Portugal \\ 
Cuba & Romania \\ 
Denmark & Senegal \\ 
East Germany& South Africa \\
Federal Republic of Yugoslavia& Soviet Union \\
Germany& Spain \\
Greece& Switzerland \\
Hong Kong& Tunisia \\
Hungary& Turkey \\
Iran& UK \\
Ireland& USA \\
Israel& Vietnam \\
Jamaica & nan\\
\hline
\end{tabular}
\caption{List of the $38$ distinct \texttt{country} values in the {\em ratings-genres-countries} dataframe.}
\label{tab:countries-from-movielens}
\end{table}

\begin{table}[t]
\centering
\begin{tabular}{ |l|l| } 
 \hline
 Action & Horror \\ 
 Adventure & IMAX \\ 
 Animation & Musical \\ 
Children & Mystery \\ 
Comedy & Romance \\ 
Crime & Sci--Fi \\ 
Documentary & Short \\ 
Drama & Thriller \\ 
Fantasy & War \\ 
Film--Noir & Western \\
\hline
\end{tabular}
\caption{List of the $20$ distinct \texttt{genre} values in the {\em ratings-genres-countries} dataframe.}
\label{tab:genres-from-movielens}
\end{table}

{\em Dataset to Dataframes:} The \texttt{user-ratedmovies.dat} file consists of movie ratings from $2,113$ anonymised users. We load the contents of this file onto a \texttt{pandas.DataFrame} in Python with the following $3$ columns: \texttt{userID}, \texttt{movieID}, and \texttt{rating}. Each row in this dataframe contains the \texttt{rating} (a number belonging to the finite set $\{0, 0.5, 1, 1.5, \ldots, 5\}$) provided by a user with ID \texttt{userID} for a movie with ID \texttt{movieID}. The so-created dataframe (say the {\em ratings} dataframe) has a total of $855,598$ rows. Along similar lines, the {\em countries} dataframe is formed from the file \texttt{movie-countries.dat} with its  columns as \texttt{movieID} and \texttt{country}. Each row of this dataframe contains information about the \texttt{country} in which a movie with ID \texttt{movieID} was shot. This dataframe has a total of $10,197$ rows, one corresponding to each \texttt{movieID}. Also, there are a total of $72$ distinct \texttt{country} names in this dataframe. Lastly, we form the {\em genres} dataframe from the \texttt{movie-genres.dat} file with its columns as \texttt{movieID} and \texttt{genre}. Each row of this dataframe contains information about the \texttt{genre} to which a movie with ID \texttt{movieID} is associated. Because each movie may be associated to one or more genres, we observe that certain \texttt{movieID}s appear in multiple rows in this dataframe, with each row corresponding to a distinct \texttt{genre}. There are a total of $20,809$ rows in this dataframe, with $20$ distinct \texttt{genre} names. 

Notice that the \texttt{movieID} column is common to the {\em ratings}, {\em countries}, and {\em genres} dataframes. Leveraging this, we create one large dataframe, say the {\em ratings-genres-countries} dataframe, with $2,240,215$ rows the following $5$ columns: \texttt{userID}, \texttt{movieID}, \texttt{country}, \texttt{genre}, and \texttt{rating}. Each row of this large dataset contains the \texttt{rating} provided by a user with ID \texttt{userID} for a movie with ID \texttt{movieID} of genre \texttt{genre} and shot in the country \texttt{country}.

{\em Dataset cleanup:} We filter the rows of the dataset by (\texttt{country}, \texttt{genre}) tuple values. We set \texttt{country} and \texttt{genre} to be in one-one correspondence with ``client'' and ``local arm'' respectively, and the average of all the \texttt{rating} values corresponding to each (\texttt{country}, \texttt{genre}) pair is the mean of the client's arm. As such, we observe that some of the clients have multiple local best arms. After the elimination of such clients, the resulting {\em ratings-genres-countries} dataframe has a total of $2,040,975$ rows. This corresponds to $8,636$ distinct \texttt{movieID} values, $38$ distinct \texttt{country} values, and $20$ distinct \texttt{genre} values. Tables~\ref{tab:countries-from-movielens} and~\ref{tab:genres-from-movielens} list the distinct \texttt{country} and \texttt{genre} values respectively. 

{\em Sampling:} At each time instant, the reward from the local arm (read \texttt{genre}) of a client (read \texttt{country}) is picked uniformly at random from the pool of all \texttt{rating} values corresponding to the (\texttt{country}, \texttt{genre}) pair.

\section{Experiments on a Synthetic Dataset of Bernoulli Observations}
\begin{figure*}[t]
     \centering
     \begin{subfigure}[b]{1.275\textwidth/4}
         \centering
         \includegraphics[width=\textwidth]{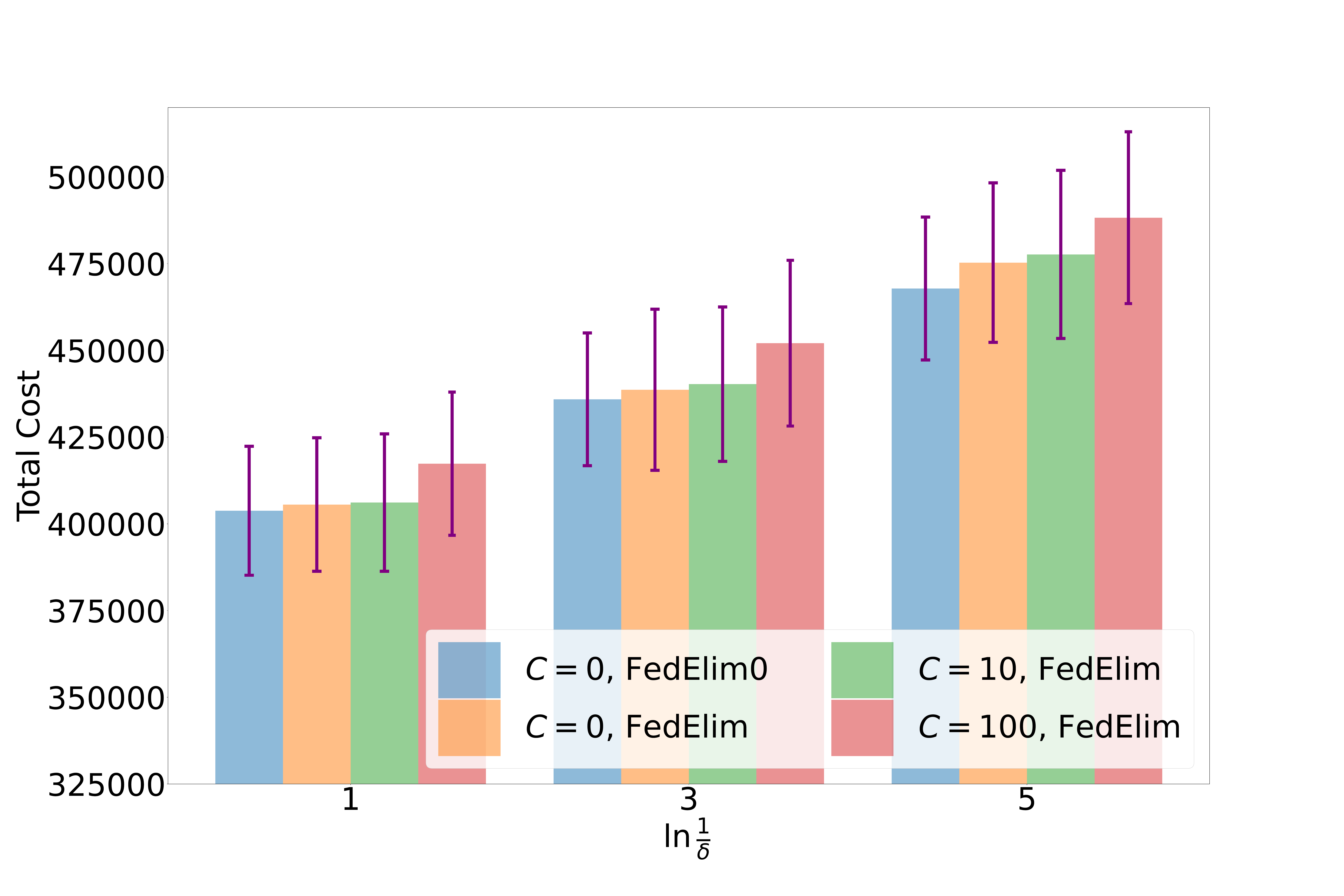}
         \caption{A plot comparing the  number of arm pulls of $\textsc{FedElim}0$ with the  cost of {\sc FedElim} for $C\in \{0, 10, 100\}$. $\textsc{FedElim}0$ has a lower cost compared to {\sc FedElim} when $C=0$.}
         \label{subfig:fedelim-total-cost-and-arm-pulls-Bernoulli}
     \end{subfigure}
     \hspace{.05in}
     \begin{subfigure}[b]{1.275\textwidth/4}
         \centering
         \includegraphics[width=\textwidth]{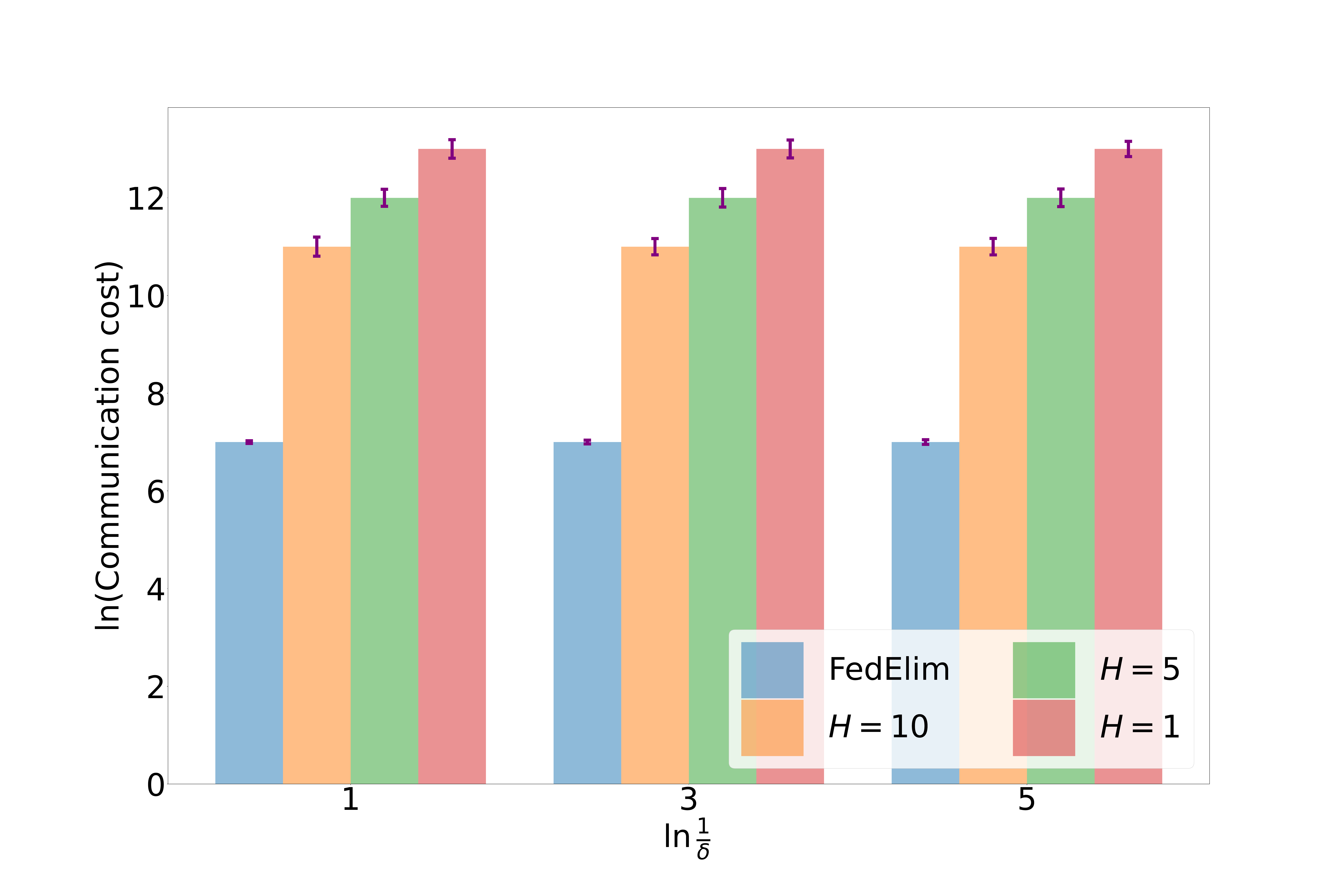}
         \caption{A   plot comparing of the communication cost  incurred under the proposed {\sc FedElim} algorithm and under periodic communication with period \mbox{$H \in \{1,5,10\}$.}}
         \label{subfig:comm-cost-comparison-Bernolli}
     \end{subfigure}
     \hspace{.05in}
     \begin{subfigure}[b]{1.275\textwidth/4}
         \centering
         \includegraphics[width=\textwidth]{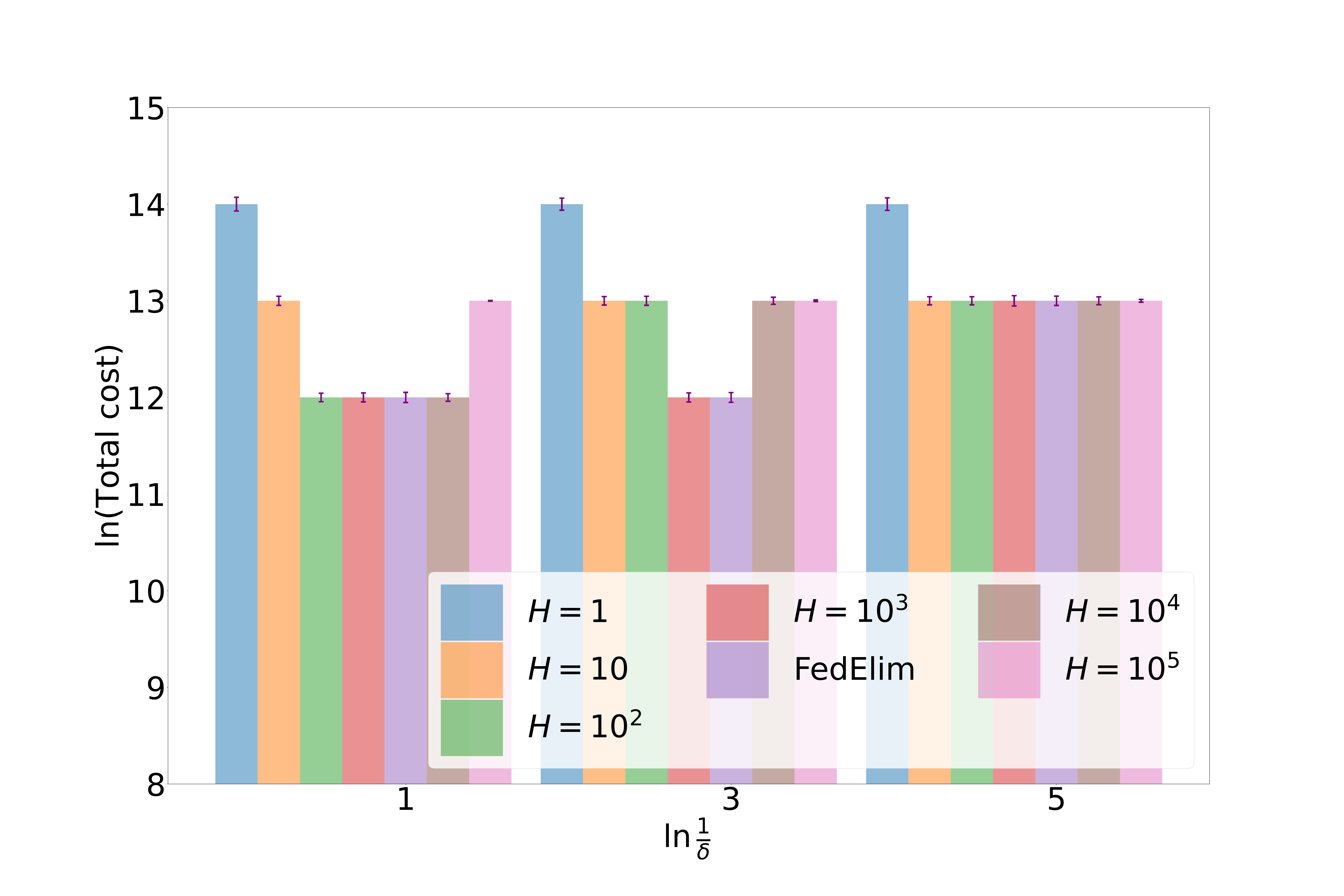}
         \caption{A   plot  comparing  the     cost     under  {\sc FedElim}   and    periodic communication with period $H =10^p$ for $p\in \{0,1,\ldots,5\}$.   {\sc FedElim} almost attains the minimum w/o knowing~$C$. }
         \label{subfig:total-cost-comparison-bernoulli}
     \end{subfigure}
        \caption{Numerical results on the synthetic Bernoulli dataset with the problem instance in \eqref{eqn:synthetic_bernoulli}.}
        \label{fig:three-figures-Bernoulli}
\end{figure*} 
To show that {\sc FedElim} correctly learns the local and global arms on a dataset that is generated according to   sub-Gaussian distributions, we run it on a dataset of Bernoulli  observations, i.e., the observations are $\{0,1\}$-valued. This dataset was considered for a tracking application in cooperative scenario \citep{mitra2021exploiting}. In particular,  we fix the following Bernoulli problem instance:
\begin{equation}
    \boldsymbol{\mu} = \begin{bmatrix}
	0.9 & 0.85 & 0.1\\
	0.85 & 0.8 & 0.3 \\
	0.7 & 0.6 & 0.5
\end{bmatrix} \in [0,1]^{3 \times 3}. \label{eqn:synthetic_bernoulli}
\end{equation}
This instance was also used in \citet{mitra2021exploiting}. 
Note that there are $3$ clients and $3$ arms and arm $k$ of client $m$ generates rewards according to Bernoulli distribution with mean $\mu_{k,m}$ for all $k\in [3]$ and $ m\in [3]$. The results are displayed in Figure~\ref{fig:three-figures-Bernoulli}. The qualitative behaviours of {\sc FedElim}, {\sc FedElim0} and the strategy that communicates with the server periodically are the same as that for the other datasets. In particular, similar to the MovieLens dataset, {\sc FedElim} again attains the absolute minimum among all schemes
that communicate periodically to the server, thus convincingly demonstrating its ability  to effectively
balance communication and the number of arm selections. Hence, our numerical results suggest that our {\sc FedElim} algorithm works well even for distributions beyond Gaussians that we analyzed in the main paper. In particular, {\sc FedElim} works effectively for distributions with bounded support, a sub-family of sub-Gaussian distributions.

\end{document}